\documentclass[10pt,journal,compsoc]{IEEEtran}

\usepackage{amsmath,amsfonts}
\usepackage{algorithmic}
\usepackage{algorithm}
\usepackage{array}

\usepackage{textcomp}
\usepackage{stfloats}
\usepackage{url}
\usepackage{verbatim}
\usepackage{graphicx}
\usepackage{cite}

\usepackage[utf8]{inputenc}
\usepackage[T1]{fontenc}
\usepackage{booktabs}
\usepackage{nicefrac}
\usepackage{xcolor}
\usepackage{multirow}
\usepackage{microtype}
\usepackage{subcaption}
\usepackage{relsize}
\usepackage{mathtools}
\usepackage{amsthm}
\theoremstyle{plain}

\newtheorem*{theorem*}{Theorem}

\usepackage{float}
\usepackage{makecell}

\usepackage{soul}
\usepackage{enumitem}

\newcommand{\cora}{C{\smaller ORA}}
\newcommand{\citeseer}{C{\smaller ITE}S{\smaller EER}}
\newcommand{\blogcatalog}{B{\smaller LOG}}

\newcommand{\wiki}{W{\smaller IKI}}
\newcommand{\flickr}{F{\smaller LICKR}}
\newcommand{\pubmed}{P{\smaller UB}M{\smaller ED}}

\begin{document}

\title{DINE: Dimensional Interpretability of Node Embeddings}

\author{Simone~Piaggesi,
        Megha~Khosla,
        Andr\'{e}~Panisson,
        and~Avishek~Anand
\IEEEcompsocitemizethanks{
\IEEEcompsocthanksitem S. Piaggesi is with University of Pisa, Italy.
\protect\\
Email: simone.piaggesi@di.unipi.it
\IEEEcompsocthanksitem A. Panisson is with CENTAI Institute, Italy.
\IEEEcompsocthanksitem M. Khosla and A. Anand are with TU Delft, Netherlands.
}
}

\maketitle

\begin{abstract}
Graphs are ubiquitous due to their flexibility
in representing social and technological systems as networks of interacting elements. 
Graph representation learning methods, such as 
node embeddings, are powerful approaches to map nodes 
into a latent vector space, allowing their use for various graph
tasks. Despite their success, only 
few studies have focused on explaining node embeddings locally. Moreover, global explanations of node embeddings remain unexplored, limiting interpretability and debugging potentials.
We address this gap by developing human-understandable explanations for dimensions in node embeddings.
Towards that, we first develop new metrics that measure the global interpretability of embedding vectors based on the marginal contribution of the embedding dimensions to predicting graph structure. 
We say that an embedding dimension is more interpretable if it can faithfully map to an understandable sub-structure in the input graph - like community structure. 
Having observed that standard node embeddings have low interpretability, we then introduce \textsc{Dine} (Dimension-based Interpretable Node Embedding), a novel approach that can retrofit existing node embeddings by making them more interpretable without sacrificing their task performance.  
We conduct extensive experiments on synthetic and real-world graphs and show that we can simultaneously learn highly interpretable node embeddings with effective performance in link prediction.
\end{abstract}


\section{Introduction}
\label{sec:intro}

Node embeddings are general purpose low-dimensional, continuous vertex representations in dense vector spaces. 
These embeddings are typically learned by trying to optimize a user-defined or flexible notion of structural similarity between vertices. 
Node embeddings have proven to be mature and popular techniques with widespread applications in web and social network analysis tasks like link prediction and community detection to name a few~\cite{hamilton2020graph} due to their simplicity, and expressive power.
However, one of their shortcomings is the innate lack of interpretability of the latent vector spaces they exist in. 
Specifically, each of the learned latent dimensions does not have a corresponding realizable interpretation in the input graphs~\cite{ liu2018interpretation, dalmia_towards_2018, gogoglou_interpretability_2019}.
This paper aims to fill this critical gap by proposing a method to retrofit ``already learned'' non-interpretable embeddings into a new and interpretable vector space without compromising the task performance.

The meaning of individual latent embedding dimensions is hard to define and determine~\cite{liu2018interpretation, dalmia_towards_2018, gogoglou_interpretability_2019}. 
We operate on a general, yet powerful notion of grounding the interpretation of each dimension to understandable sub-structures of the input graphs, e.g., communities, subgraphs, etc.
This design decision has clear advantages in downstream tasks where sub-graphs or communities are clear explanations of a predictive task.
As a concrete example, in a link-prediction task the likelihood of a link is higher for a pair of nodes in the same community \cite{cannistraci2013link}. 
Similarly, in several bio-medicals tasks that use embedding features like~\cite{embeddings_bio,dong2022message}, subgraphs refer to a protein or genetic pathways.
Therefore, automatically grounding latent dimensions to sub-graphs and community structures will improve understanding of the prediction process.

Existing literature investigating the interpretability of node embeddings is limited to three major aspects. 
First, posthoc feature-attribution methods like~\cite{ying2019gnnexplainer,funke2022:zorro,vu2020pgm} explain a decision in terms of (a) subset of node features or (b) nodes/edges in the computational graph. If embeddings are used as features, subsets of latent features are still non-interpretable. 
Explaining a prediction in terms of edges and nodes is a useful first step, but these approaches cannot be used globally. Specifically, the meaning of a latent dimension can still not be explained using these local methods.
Secondly, in the presence of ground-truth node labels \cite{gogoglou_interpretability_2019, duong2019interpretable} measure the interpretability of embedding dimensions in terms of their association strength with these labels without providing an explicit explanation like our approach.
These approaches pre-suppose a certain interpretable mapping and are not flexible.
Finally, \cite{idahl2020finding} search for interpretable subspaces related to concepts from knowledge bases. Unlike earlier works, we operate on a more generalized setting extracting explanations that are agnostic to ground truth labels. Additionally, and more importantly, we propose methods that retrofit existing node embeddings to make their dimensions more interpretable.

In this work, the central aim is finding the human-interpretable meaning of node embedding dimensions and associating latent directions with understandable structural features of the input graph. Despite this \textit{post-hoc} approach, we cannot rely on existing tools for interpreting prediction models \cite{ribeiro2016should, lundberg2017unified}, mainly because they are designed for supervised tasks. Instead, we analyze unsupervised node embeddings to find evidence for interpretable latent units associated with individual semantic concepts of the input data. 
Since many empirical graphs are characterized by an underlying community structure~\cite{girvan2002community}, we associate communities as interpretable semantic concepts of the input data. 
We first develop a metric for interpreting each dimension of the node embeddings in terms of its utility in predicting  edges in the graph.

\begin{figure*}[h]
     \centering

     \begin{subfigure}{\linewidth}
         \centering
         \includegraphics[width=\textwidth]{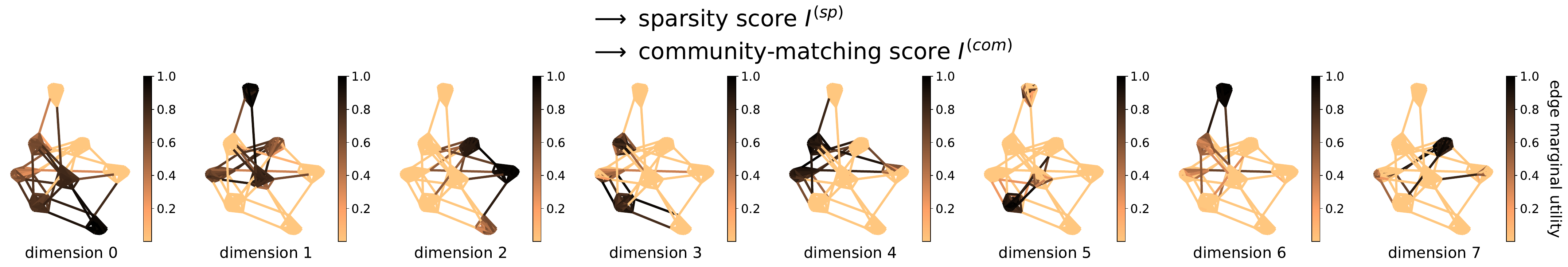}
         \caption{\textsc{DeepWalk}}
     \end{subfigure}
     \begin{subfigure}{\linewidth}
         \centering
         \includegraphics[width=\textwidth]{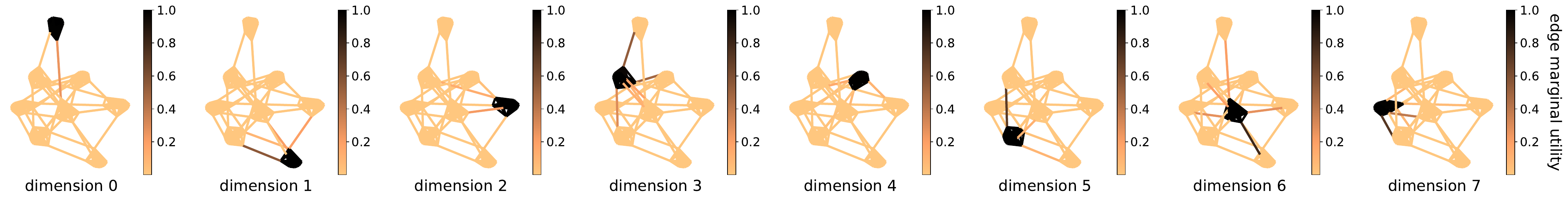}
         \caption{\textsc{Dine}}
     \end{subfigure}
    \caption{Saliency plots with edge marginal utilities (normalized between 0 and 1) for 8-dimensional embeddings trained on a graph generated via stochastic block model~\cite{holland1983stochastic}, with 80 nodes divided into 8 cliques.
    In picture (a), from the left to the right, \textsc{DeepWalk} dimensions are sorted by increasing their matching level with individual cliques, and with growing sparsity of the per-dimension subgraphs. 
    In picture (b) our approach \textsc{Dine} is able to obtain interpretable dimensions both in terms of community-matching and sparsity.}
    \label{fig:heatmaps}
\end{figure*}
Our utility measure -$\mu_d(\mathbf{u},\mathbf{v})$- is based on feature removal methods and expresses the individual contribution of dimension $d$ for predicting 
edge $(u,v)$ from embeddings. 
Using this measure, we can construct saliency maps (refer to Figure~\ref{fig:heatmaps}), to recognize groups of edges (salient subgraphs) that are reconstructed by specific dimensions. We then define quantitative metrics to estimate the interpretability of latent dimensions:  we assess interpretability as a ``degree of association'' with individual graph communities and the sparsity level for these associations. 
As an example, in Figure~\ref{fig:heatmaps}(a), we sort \textsc{DeepWalk} dimensions according to our metrics of interpretability, showing that
the majority of units are not immediately interpretable since they do not match with a single clique of the synthetic graph.

Secondly, we propose a novel and modular approach, called \textsc{Dine}, to post-process existing node representations and enhance their interpretability. 
DINE embeds input embeddings into a new sparse, interpretable, and low-entropy vector space. 
Figure~\ref{fig:heatmaps}(b) shows the result of this post-processing step on \textsc{DeepWalk} vectors. DINE also preserves the topological graph information to be employed in usual downstream tasks with minimal performance loss. 

In our extensive experimental evaluations, we compare our approach with other embedding methods in terms of interpretability, link prediction performance, and scalability over multiple real-world graph datasets. Our results show that DINE convincingly, under most experimental conditions, outperforms existing baselines in terms of dimensional interpretability with negligible to no performance losses.
To summarize,  our main contributions are as follows:

\begin{itemize}
  \item We formalize the desirable properties for global explanations of node embeddings, namely \textit{decomposability, comprehensibility, and sparsity}.
  
    \item 
    We introduce a new utility measure that allows the extraction of explanatory subgraphs, one for each dimension ( addressing the property of Decomposability). Our measure is  grounded in feature attribution techniques like the Shapley value, which are widely used for removal-based explanations. Moreover, we propose two metrics to measure Comprehensibility and Sparsity of explanatory subgraphs.

    \item 
    We propose a novel, modular, and theoretically sound method DINE that intends to retrofit existing node embeddings to improve their global interpretability.

    \item 
    We run extensive experimental analyses to establish clear gains in the interpretability-performance trade-offs using DINE.

\end{itemize}

Our code and artifacts will be released 
at \url{https://www.github.com/simonepiaggesi/dine}.

\section{Preliminaries and Related Work}

\subsection{Preliminaries and Notation}
Given an undirected, unattributed and unweighted graph $\mathcal{G} = (\mathcal{V}, \mathcal{E})$, node embeddings  are the output of an encoder function 
\begin{equation}
e: v \in \mathcal{V} \mapsto \mathbf{e}(v) = \mathbf{v} \in \mathbb{R}^D
\end{equation}
which map nodes into geometric points of the $D-$dimensional vector space $\mathbb{R}^D$ (usually $D << |\mathcal{V}|$). We will refer to both $\mathbf{e}(v)$ and $\mathbf{v}$ as the embedding vectors of node $v \in \mathcal{V}$ mapped through the encoder $e$, and with  $e_d(v)$ and $\mathrm{v}_d$ as entries of these vectors corresponding to dimension $d$. Usually embedding vectors are collected into the columns of the embedding matrix $\mathbf{X} \in \mathbb{R}^{D\times |\mathcal{V}|}$, i.e. $X_{d,v} = \mathrm{e}_d(v) = \mathrm{v}_d$.
Later we will refer to $D$ as the cardinality of the set $\mathcal{D}=\{1,\dots, D \}$ containing the enumerated dimensions.

In the case of \textsc{DeepWalk} \cite{perozzi2014deepwalk} and \textsc{Node2Vec} \cite{grover2016node2vec}, the encoder is a lookup function where node representations are learned 
through the optimization of a neighborhood reconstruction loss. Specifically, the  output of a decoder $\mathrm{DEC}:\mathbb{R}^D \times \mathbb{R}^D \rightarrow \mathbb{R}$ is optimized to predict node pairs $(u,v) \in \mathcal{T} \subseteq \mathcal{V \times \mathcal{V}}$ generated from co-occurrences in unbiased or biased random walks.
Many other embedding methods fit this \textit{encoder-decoder} framework \cite{hamilton2017representation}: for example, factorization-based embeddings \cite{ou2016asymmetric, cao2015grarep} and even deep neural networks methods like \cite{kipf2016variational}, where the encoder function is given by a graph convolutional network \cite{zhang2019graph}.

\subsection{Related Work}

\textbf{Interpretability for node embeddings.} From the node embeddings perspective, interpretability is a multi-faceted concept that has been studied from different angles.
In \cite{bonner_exploring_2019, dalmia_towards_2018} authors investigate, using prediction tasks, whether specific topological graph features (e.g. degree centrality, clustering coefficient, etc.) are encoded into node representations. These works
significantly differ from our approach, where the aim is to find the comprehensible meaning of embedding dimensions, associating single dimensions with interpretable graph structures (e.g. communities).
Other methods focus on measuring interpretability of node embeddings with respect to node labels \cite{gogoglou_interpretability_2019} and node centralities \cite{khoshraftar_centrality-based_2021}. In~\cite{liu2018interpretation} global interpretations are given as a hierarchy of graph partitions, but they do not focus on interpreting single dimensions.
In \cite{wang2019discerning} the authors study the impact on learned node embeddings when removing edges from the input graph. Instead, \cite{park2022providing} estimates the importance of candidate nodes in each node representation. Another line of research focuses on producing interpretable-by-design representations based on graph clustering \cite{duong2023deep, rozemberczki_gemsec_2019}, which are conceptually analogous to community-preserving node embeddings \cite{wang_community_2017}.

\textbf{Interpretability for link prediction.} Our approach focuses on the interpretation of embedding dimensions according to the graph structural reconstruction task, and it is related to methods for the interpretability of embedding-based link prediction. For instance, ExplaiNE \cite{kang2019explaine} quantifies the variation in the probability of a link when adding or removing neighboring edges. PaGE-Link \cite{zhang2023page} generates explanations as paths connecting a node pair, while ConPI
\cite{wang2021modeling} provides the most influential interactions computed with an attention mechanism over the contextual neighborhoods. Other relevant methods study the problem of explaining link prediction in knowledge graphs~\cite{rossi2022explaining, zhang2019interaction}. We, on the other hand, aim to explain the node embedding itself by associating explanations with each of its dimensions. 

\textbf{Interpretability for word embeddings.} 
Since many methods for graph representation learning are based on language models, techniques for interpreting the dimensions of word embeddings are also relevant for node embeddings.
Previous literature in this area focuses on interpreting the dimensions of word embeddings based on semantic information \cite{prouteau2022embedding, senel_semantic_2018}, or analyzing geometric properties of the embedding space \cite{shin2018interpreting, park_rotated_2017}. Several other studies also propose approaches to learn interpretable representations by design, where the goal is
achieving \textit{sparsity}~\cite{liang_anchor_2021, sun2016sparse, luo_online_2015}. However, due to the high popularity of some embedding approaches, post-processing techniques that are built upon these approaches have been often preferred rather than interpretable-by-design methods \cite{subramanian_spine_2018, chen_kate_2017, faruqui_sparse_2015}.

\section{Explaining Node Embeddings}
\label{sec:utility}

We start by formalizing the desired  fundamental properties of a global explanation for node embeddings.
 Intuitively, as graph structure serves as the input for generating unsupervised node embeddings, we seek reliable explanations in terms of associations between model parameters and human-understandable units of the input graph.

\begin{description}
   \item[Decomposability]   A global explanation should be able to refer to single parts of the model, and then explain these parts individually \cite{belle2021principles}. This is different from local instance-based explanation, where the focus is to interpret the result on single node predictions. In particular, a global explanation for node embeddings should be able to explain separately each dimension of the embedding space. To do so, in this work we extract interpretations in the form of important subgraphs $\mathcal{G}_d$ that we identify as the ``meaning'', or ``explanation'', of a dimension $d$.  

   \item[Comprehensibility] An explanation should be human-understandable, in the sense that it relates to meaningful graph features \cite{ dalmia_towards_2018} or discernible concepts \cite{idahl2020finding}. With subgraph-based explanations, such features can be seen as structural components that we identify with the \textit{communities} of the graph. For instance in biological networks like protein-protein interaction networks, these subgraphs could be important pathways responsible for biological mechanisms associated with for example a protein function or disease progression. In other graphs such as social networks these subgraphs can be seen as communities.  Communities are typically considered as one of the fundamental organizing principles in these graphs \cite{girvan2002community} justifying their choice to identify the meaning of representation dimensions.

   \item[Sparsity] Explanations should be associated only with a minimal set of graph elements that  sufficiently explain the learned parameters, ignoring the irrelevant ones \cite{liang_anchor_2021, sun2016sparse}. In our case, sparsity quantifies the spatial localization of an explanation subgraph
   
\end{description}

Having defined the desired properties, we next describe how to obtain such decomposable explanations for node embeddings. In Section \ref{sec:interpret_scores} we  propose new metrics to quantify both the comprehensibility and the sparsity of these explanations.

\subsection{Decomposable explanations}
\label{sec:dim_subgraphs}

Here we describe how we obtain global and decomposable  explanations of node embeddings by extracting one explanation for every dimension of the latent space. Intuitively, given that embeddings are typically optimized for graph structure prediction, we aim to uncover 
the importance of individual dimensions in reconstructing the sub-structures of the graph. These substructures, consequently, will serve as explanations for individual dimensions. 
To extract the substructure explanations, we develop a 
utility function $\mu_d(\mathbf{u},\mathbf{v})$ which quantifies the dimension's contribution in reconstructing a single graph edge with an embedding decoder. In fact, the score returned by the decoder $\mathrm{DEC}(\mathbf{u},\mathbf{v})$ can be used to perform edge reconstruction, i.e. assessing the existence of edges $(u,v)$: the higher the score, the higher the likelihood of observing the link on the input graph. 

Here we adopt a simple yet effective approach  for attributing dimension importance based on feature removal \cite{datta2016algorithmic, li2016understanding}. 
Specifically, we define the attribution score of a single dimension $d \in \mathcal{D}$ in the reconstruction of an edge $(u,v)$ as:
\begin{equation}
    \label{eq:mu_d}
    \mu_d(\mathbf{u},\mathbf{v}) = \Delta_{\mathcal{D}}(\mathbf{u},\mathbf{v})-\Delta_{\mathcal{D}\setminus\{d\}}(\mathbf{u},\mathbf{v}),
\end{equation}
where $\Delta_\mathcal{S}:\mathbb{R}^{|\mathcal{S}|} \times \mathbb{R}^{|\mathcal{S}|} \rightarrow \mathbb{R}$ quantifies the average edge scoring of dimensions in the subset  $\mathcal{S} \subseteq \mathcal{D}$
\begin{equation}
    \Delta_{\mathcal{S}}(\mathbf{u},\mathbf{v}) = \frac{1}{|\mathcal{S}|}\sum_{d \in \mathcal{S}} \mathrm{u}_d \mathrm{v}_d.
\end{equation}
Notably, we consider a product-based scoring function that is appropriate to work with popular methods such as \textsc{DeepWalk} and \textsc{Node2Vec}. For an individual edge, the function in Eq.~(\ref{eq:mu_d}) measures how much the average likelihood increases or decreases when removing dimension $d$ from the whole set $\mathcal{D}$. From a game-theoretic point of view, the importance scores $\mu_d(\mathbf{u},\mathbf{v})$ defined above is an example of \textit{marginal utility}, which expresses the contribution of dimension $d$ when it is added to the \textit{coalitional set} $\mathcal{S} = \mathcal{D} \setminus \{d\}$. A more exhaustive computation takes into account the average marginal contribution according to any possible coalitional set $\mathcal{S} \subset \mathcal{D}$ and it is given by the Shapley value \cite{shapley201617} score:
\begin{equation}
\label{eq:shapley}
\phi_d(\mathbf{u},\mathbf{v}) = \sum_{\mathcal{S} \subseteq \mathcal{D}\setminus\{d\}} \frac{\binom{|\mathcal{D}|-1}{|\mathcal{S}|}^{-1}}{|\mathcal{D}|}\left[\Delta_{\mathcal{S}\cup \{d\}}(\mathbf{u},\mathbf{v})-\Delta_{\mathcal{S}}(\mathbf{u},\mathbf{v})\right],
\end{equation} 
where the difference $\Delta_{\mathcal{S}\cup \{d\}}(\mathbf{u},\mathbf{v})-\Delta_{\mathcal{S}}(\mathbf{u},\mathbf{v})$ corresponds to the marginal utility of adding $d$ to the dimensions' coalition $\mathcal{S} \subset \mathcal{D}$. Therefore, the importance score $\mu_d(\mathbf{u},\mathbf{v})$ corresponds to the marginal utility given by Eq.~\eqref{eq:shapley} with respect to the maximal coalitions ($|\mathcal{S}| = |\mathcal{D}|-1$).

Since the exact computation of \eqref{eq:shapley} has exponential time complexity, several approximation methods have been proposed in the literature to address scalability issues \cite{vstrumbelj2014explaining, castro2009polynomial}. Additionally, most of the approximations assume independence among features \cite{vstrumbelj2014explaining, lundberg2017unified} and suffer from considering feature correlations \cite{aas2021explaining}.
Rather than introducing an approximation, the marginal utility adopted here helps to derive computationally feasible formulas: in fact, the computation of $\mu_d$ reduces the time complexity from  the order of $2^{|\mathcal{D}|-1} |\mathcal{D}||\mathcal{E}|$ to $|\mathcal{D}||\mathcal{E}|$ when computed over all edges of the graph. Moreover, since mutual independence of features is not usually guaranteed for node embeddings, the simplification is due to express the effect of isolated dimensions disregarding possible feature correlations.

We use the importance scores defined in \eqref{eq:mu_d} to determine the explanation subgraphs formed by the edges that benefit most from the presence of a dimension $d$. 
Specifically, we identify the subgraph $\mathcal{G}_d$ induced by links $\mathcal{E}_d = \{ (u,v) \in \mathcal{E} : \mu_d(\mathbf{u},\mathbf{v})>0 \}$ with \textit{positive} marginal utility as the explanation of dimension $d$.
We restrict ourselves to positive payoffs because the main interest is to find those dimensions which are more effective in predicting a given edge, leaving for future work the analysis of the negative effects. 
In Figure \ref{fig:karate} we highlight the explanation subgraphs in the \textsc{Karate-Club} dataset for 2-dimensional \textsc{DeepWalk} embeddings.
\begin{figure}
    \centering
    \includegraphics[width=0.4\textwidth]{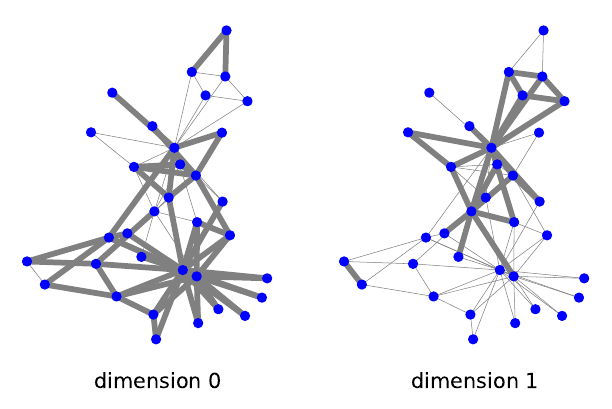}
    \caption{Utility induced subgraphs for 2-dimensional \textsc{DeepWalk} embeddings trained on \textsc{Karate-Club}.}
    \label{fig:karate}
\end{figure}
We say that subgraphs depicted in Figure~\ref{fig:karate} are global explanations of \textsc{DeepWalk} because they allow associating any model dimension with pieces of the data, and provide a global interpretation that is decomposed into per-dimension views.

\subsection{Measuring  comprehensibility and sparsity}
\label{sec:interpret_scores}

Here we define metrics to quantify the quality of the extracted subgraph explanations.
Specifically, we introduce two interpretability metrics to measure the comprehensibility and sparsity of the per-dimension induced subgraphs.

\textbf{Community-aware metric.} 
Let $\mathcal{E}_d$ denote the set of edges in the explanation subgraph for dimension $d$. Given the information on important subgraphs for example pathways in case of biological networks or communities in social networks, we measure the relevance of explanation subgraphs to these communities/subgraphs using precision and recall scores.
Let $\mathcal{P} = \{\mathcal{P}_1,\dots,\mathcal{P}_n \}$ denote the set of ground-truth link partitions/communities/subgraphs of the input graph.  
Later in the experiment section, we will also describe how to obtain such ground-truth subgraphs  when these are not given. With the membership function $m : \mathcal{E} \rightarrow \mathcal{P}$, we first compute precision and recall metrics which measure the association strength of extracted explanation subgraphs with the given ground-truth important subgraphs/communities.
\begin{equation}
    \mathrm{precision}(\mathcal{E}_d,\mathcal{P}_i) = \frac{| \{ (u,v) \in \mathcal{E}_d : m(u,v)=\mathcal{P}_i \} |}{|\mathcal{E}_d|}
\end{equation}
\begin{equation}
    \mathrm{recall}(\mathcal{E}_d,\mathcal{P}_i) = \frac{| \{ (u,v) \in \mathcal{E}_d : m(u,v)=\mathcal{P}_i \} |}{| \mathcal{P}_i |}
\end{equation}
We then compute the interpretability score $I_d$ as the maximum F1 score over all given ground-truth communities.
\begin{equation}
    \hat{\mathcal{P}}_d = \underset{\mathcal{P}_i \in \mathcal{P}}{\mathrm{argmax}} ~\mathrm{F1}(\mathcal{E}_d,\mathcal{P}_i) ; \hspace{10mm}
    I^{(com)}_d = \underset{\mathcal{P}_i \in \mathcal{P}}{\mathrm{max}} ~\mathrm{F1}(\mathcal{E}_d,\mathcal{P}_i)
\end{equation}
where F1-score is the harmonic mean between $\mathrm{precision}(\mathcal{E}_d,\mathcal{P}_i)$ and $\mathrm{recall}(\mathcal{E}_d,\mathcal{P}_i)$. Higher values of $I_d$ indicate the dimension $d$ is strongly associated with a single community. Global community-aware interpretability can be quantified with the average $I^{(com)} = \frac{1}{|\mathcal{D}|}\sum_{d \in \mathcal{D}}I^{(com)}_d$.

\textbf{Sparsity-aware metric.} 
In the absence of ground-truth community information, we can anyhow quantify in an unsupervised manner whether dimensions can highlight structure-relevant subgraphs. In particular, without any cognition on community structure, it is highly preferable that interpretable directions of the embedding space are associated with a minimal set of significant edges. Inspired by explanation masks in graph neural networks \cite{yuan2020explainability}, we formulate its calculation using Shannon entropy \cite{funke2022z}:

\begin{equation}
    I^{(sp)}_d = - \frac{1}{\log |\mathcal{E}|}\sum_{(u,v) \in \mathcal{E}} \left(\frac{[(u,v) \in \mathcal{E}_d]}{z_d}\right) \log \left(\frac{[(u,v) \in \mathcal{E}_d]}{z_d}\right)
\end{equation}
where the function $[*]$ returns 1 if the proposition inside is true (and 0 otherwise), and $z_d=\sum_{(u,v) \in \mathcal{E}}[(u,v) \in \mathcal{E}_d]$ is a normalization for the correct computation of the Shannon entropy. Lower values indicate that embedding dimensions are associated with smaller-sized subgraphs.
Global sparsity-aware interpretability can be quantified with the average $I^{(sp)} = \frac{1}{|\mathcal{D}|}\sum_{d \in \mathcal{D}}I^{(sp)}_d$.

\section{Our Approach: Dimension-based interpretable node embedding}
\label{sec:dine_method}
\begin{figure*}[h!]
    \centering\includegraphics[width=0.8\textwidth, trim={0cm 1.cm 0cm 2.cm},clip]{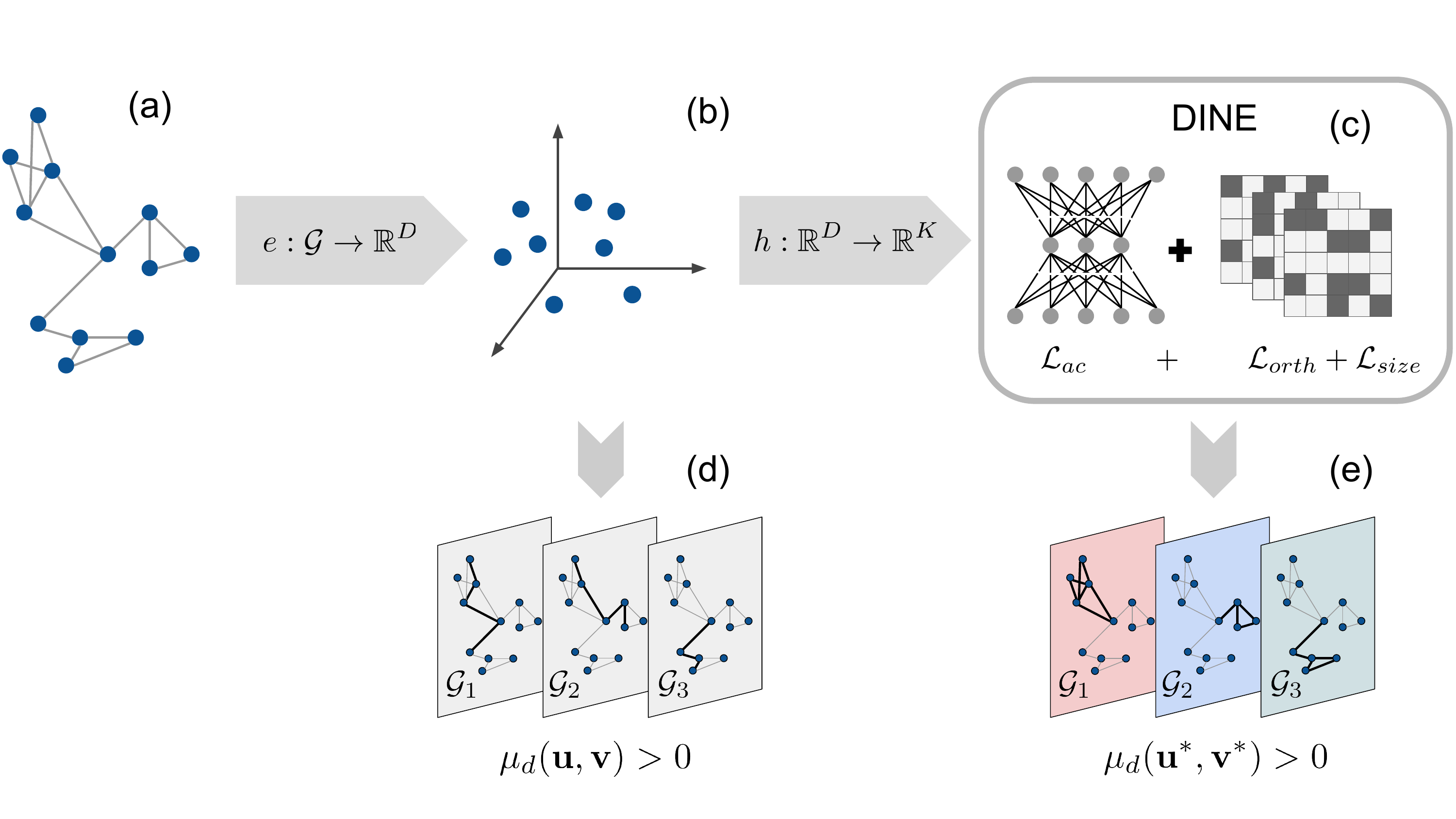}
    \caption{Schematic view of our methodology. Starting from a graph (a), we consider a given embedding representation (b), to which we apply the method DINE (c). Explanations are given in the form of per-dimension subgraphs, both for the starting embedding (d) and the DINE embedding (e).}
    \label{fig:my_label}
\end{figure*}
In previous sections, we proposed the utility-induced subgraphs as explanations to interpret node embedding dimensions.  Unsurprisingly, as we show in Figure~\ref{fig:heatmaps}(a) for \textsc{DeepWalk}, typically it is difficult to map utility-induced subgraphs to interpretable graph units, mainly because these methods are trained with the unique goal of maximizing reconstruction performance. Filling this gap, we introduce \textit{Dimension-based Interpretable Node Embedding} (DINE), a novel method to improve the interpretability of already trained node embeddings by retrofitting the induced subgraphs which affect interpretation metrics (\S\ref{sec:interpret_scores}). 

We design such retrofit task with an autoencoder architecture, trained to reconstruct embedding vectors in the input \cite{baldi2012autoencoders}. By encoding the input node representations into a hidden feature space, the autoencoder can be regularized in order to promote the learning of interpretable dimensions. Despite the many existing regularizations already used to obtain interpretable embeddings, such as non-negativity \cite{yang2013overlapping, luo_online_2015} or sparsity \cite{subramanian_spine_2018, sun2016sparse}, in this work we employ \textit{orthogonality} regularizers \cite{bansal2018can, massart2022orthogonal, schaaf2019enhancing} to achieve the purpose. Orthogonality is closely related with \textit{disentaglement} \cite{higgins2018towards, higgins2016beta}, which is a key concept implemented in several methods for decoupling correlations between latent dimensions \cite{cha2023orthogonality, song2022orthogonal}, with the results of learning more compact representations whose feature dimensions are associated with independent facets of data.

We argue improving dimension-based orthogonality is more effective for several reasons:
\begin{itemize}
    \item \textit{Distinct features}:  If two dimensions are orthogonal, it means that they are independent of each other and do not share any latent factor. Therefore, each dimension provides unique information which can be interpreted as representing a distinct characteristic of the data.
    
    \item \textit{Separation of concepts}: Orthogonal dimensions in the embedding space can represent independent concepts. For example, in the context of word embeddings, the concept of "gender" might be captured along one dimension, while the concept of "age" might be captured along another. This helps us to easily separate and understand these different features of the data.
    
    \item \textit{Removing redundancy}: Orthogonality implies no redundancy. If two dimensions are not orthogonal, then they project onto each other to some extent, meaning there's some shared information. This shared information could be interpreted as redundancy. By ensuring orthogonality, we ensure that each latent direction provides new, unique information.

    \item \textit{Clear interpretation of distances:} In an orthogonal space, distances directly correlate with dissimilarity. For instance, two orthogonal word embeddings would likely represent words with very different meanings or features, while vectors closer together would represent more similar words.    
\end{itemize}

Contrary to previous works \cite{schaaf2019enhancing} that enforce orthogonality of neural weights, here we employ orthogonalization of the edge reconstruction patterns that directly affect per-dimension utility subgraphs. In this way, we obtain node embeddings whose interpretability is optimized according to  metrics introduced in previous section.
In the next, we first show how we can rephrase the utility optimization in an effective way to be easily handled, and how the method is implemented.

\subsection{Optimization of Marginal Utilities}

DINE aims to learn an opportune mapping $h: \mathbb{R}^{D} \rightarrow \mathbb{R}^{K}$, in such a way that per-dimension subgraphs are highly interpretable in terms of decomposability, comprehensibility and sparsity of explanations. 
We use both $\mathbf{v}^*$ and $h(\mathbf{v})$ to indicate the embedding vectors of node $v \in \mathcal{V}$ mapped with the embedding function $h$, and collected into the matrix $\mathbf{H} \in \mathbb{R}^{K \times |\mathcal{V}|}$. We also refer to $K$ as the cardinality of the set containing the enumerated dimensions of the new space $\mathcal{K}=\{1,\dots K\}$. 

Since subgraphs are the results of positive marginal utilities $\mu_d(\mathbf{u}^*, \mathbf{v}^*)$, we are interested in optimizing the utility measures as a function of the new embedding parameters defined by $h$. In the following paragraphs we show that, assuming the new embedding space to be the unit-size hypercube $[0,1]^{K} \subset \mathbb{R}^{K}$, together with sufficiently high embedding dimensionality, we can simplify the optimization of the utility measure.
In fact, in the following theorem we show that for a given edge, the interpretability (utility) measure for a dimension can be approximated using a single dot product over the embedding pair. 
\begin{theorem*}
    Let be $h: \mathbb{R}^{D} \rightarrow [0,1]^{K}$ the mapping from an embedding encoder of $\mathcal{G} = (\mathcal{V}, \mathcal{E})$ and the $K$-dimensional hypercupe. For high dimensionality $K$, the per-dimension utility score for edge $(u,v) \in \mathcal{E}$, $\mu_d(\mathbf{u}^*,\mathbf{v}^*)$, can be expressed as:
    \begin{equation*}
    \mu_d(\mathbf{u}^*,\mathbf{v}^*) + \mathcal{O}(\frac{1}{K^2}) = \frac{\mathrm{u}^*_d\mathrm{v}^*_d}{K}
    \end{equation*} 
\end{theorem*}

\begin{proof}
    We start writing the formula for $\mu_d(\mathbf{u}^*,\mathbf{v}^*)$:
 \begin{equation*}
  \mu_d(\mathbf{u}^*,\mathbf{v}^*) = \frac{1}{K}\sum_{q \in \mathcal{K}} \mathrm{u}^*_q\mathrm{v}^*_q - \frac{1}{K-1}\sum_{q \in \mathcal{K}\setminus \{d\}} \mathrm{u}^*_q\mathrm{v}^*_q  
\end{equation*}
Using the expression $\sum_{q \in \mathcal{K}} \mathrm{u}^*_q\mathrm{v}^*_q = \mathbf{u}^* \cdot \mathbf{v}^* = \mathrm{u}^*_d\mathrm{v}^*_d + \sum_{q \in \mathcal{K}\setminus \{d\}} \mathrm{u}^*_q\mathrm{v}^*_q $, we find that:
\begin{align*}
  \mu_d&(\mathbf{u}^*,\mathbf{v}^*) = \frac{1}{K} \mathrm{u}^*_d\mathrm{v}^*_d - \left(\frac{1}{K-1} - \frac{1}{K}\right)\sum_{q \in \mathcal{K}\setminus \{d\}} \mathrm{u}^*_q\mathrm{v}^*_q  \\
  &= \frac{1}{K} \mathrm{u}^*_d\mathrm{v}^*_d - \frac{1}{K} \left(\frac{1}{1-\frac{1}{K}} -1\right) \left( \mathbf{u}^*\cdot\mathbf{v}^*-\mathrm{u}^*_d\mathrm{v}^*_d\right)
\end{align*}
Ignoring the case of 1-dimensional embeddings, $K>1$ and $0<\frac{1}{K}<1$, then the expression in the first parenthesis can be rewritten using the geometric series:
\begin{equation*}
\frac{1}{1-\frac{1}{K}}-1 = \sum_{l=0}^\infty \frac{1}{K^l} - 1 = \frac{1}{K} + \mathcal{O}(\frac{1}{K^2})
\end{equation*}
replacing the formula with this last result, we have:

\begin{align*}
K\cdot&\mu_d(\mathbf{u}^*,\mathbf{v}^*) = \mathrm{u}^*_d\mathrm{v}^*_d - \left( \frac{1}{K} + \mathcal{O}(\frac{1}{K^2}) \right)\left( \mathbf{u^*}\cdot\mathbf{v}^*-\mathrm{u}^*_d\mathrm{v}^*_d\right)\\
&= \mathrm{u}^*_d\mathrm{v}^*_d \left(1+ \frac{1}{K} + \mathcal{O}(\frac{1}{K^2})\right) - \mathbf{u}^*\cdot\mathbf{v}^* \left( \frac{1}{K} + \mathcal{O}(\frac{1}{K^2})\right)\\
&= \big( \mathrm{u}^*_d\mathrm{v}^*_d -\frac{1}{K} \mathbf{u}^* \cdot \mathbf{v}^* \big) \left(1+\frac{1}{K}+\mathcal{O}(\frac{1}{K^2})\right) 
\end{align*}

In the case of the function $h$, we have $\mathrm{u}_d^*\cdot \mathrm{v}_d^* \in [0,1]$ and so $\mathbf{u}^* \cdot \mathbf{v}^* = \alpha_{uv} K$, where $0 < \alpha_{uv} < 1$. With sufficient high dimensionality $K$, we can assume that values $\mathrm{u}_d^*\cdot \mathrm{v}_d^*$ are independent from the dimensionality $K$, then $\alpha_{uv} \propto \frac{1}{K}$.
Therefore, neglecting infinitesimal terms in parenthesis we obtain $\mu_d(\mathbf{u}^*,\mathbf{v}^*) +\frac{\alpha_{uv}}{K} +\mathcal{O}(\frac{1}{K^2}) = \mu_d(\mathbf{u}^*,\mathbf{v}^*)  +\mathcal{O}(\frac{1}{K^2}) \approx  \frac
{\mathrm{u}^*_d\mathrm{v}^*_d}{K}$.
\end{proof}

We define $\frac
{\mathrm{u}^*_d\mathrm{v}^*_d}{K}$ as the entries $m^*_d(u,v)$ of $K$ continuous-valued graph masks $\{ \mathbf{M}_d \in \mathbb{R}^{|\mathcal{V}| \times |\mathcal{V}|} \}_{d\in \mathcal{K}}$, 
that can be easily computed as the outer products of the rows of $\mathbf{H}$, namely $m_d^*(u,v) = (\mathbf{H}_{d,:} \otimes \mathbf{H}_{d,:})_{uv}$.

The equation
$\mu_d(\mathbf{u}^*,\mathbf{v}^*)  +\mathcal{O}(\frac{1}{K^2}) \approx  \frac
{\mathrm{u}^*_d\mathrm{v}^*_d}{K} \equiv m^*_d(u,v)$ tells that the quantities $m^*_d(u,v)$ and  $\mu_d(\mathbf{u}^*,\mathbf{v}^*)$ differ by a negligible term $\mathcal{O}(\frac{1}{K^2})$. Consequently, we can optimize the quantities $m^*_d(u,v)$ that are computed from individual products.
Graph masks have the role of highlighting the structure-relevant edges for any direction in the $K$-dimensional learned space.

\subsection{Method Implementation}

We now describe the retro-fitting optimization task and identify the key design choices which allow obtaining more interpretable induced subgraphs, and whose effectiveness is shown in the experiments section. Please refer to Figure \ref{fig:my_label} for a schematic diagram of our approach.

We implemented $h$ as the latent projection of a single-layer autoencoder, namely $h(\mathbf{v}) = \sigma(\mathbf{W}^{(0)} \mathbf{v} + \mathbf{b}^{(0)})$, which returns $\tilde{\mathbf{v}} = \mathbf{W}^{(1)} h(\mathbf{v}) + \mathbf{b}^{(1)}$ as output.  The hidden layer matrix of the autoencoder, $\mathbf{H}\in \mathbb{R}^{K \times|\mathcal{V}|}$, collects the components $H_{d,v} = \mathrm{h}_d(\mathbf{v}) \in [0,1]$ of the interpretable embedding vectors that we aim to learn. 

We add regularization constraints on the hidden embedding matrix $\mathbf{H}$ while training the autoencoder, in order to learn optimal graph masks. Specifically, we minimize the following loss: 
\begin{equation}
    \mathcal{L} = \mathcal{L}_{ac}(\mathbf{X}, \tilde{\mathbf{X}}) + \sum \mathcal{L}_{reg}(\mathbf{H})
\end{equation}
where 
$ \mathcal{L}_{ac}(\mathbf{X}, \tilde{\mathbf{X}}) = \mathrm{MSE}(\mathbf{X}, \tilde{\mathbf{{X}}}) = \frac{1}{|\mathcal{V}|}\sum_{v \in \mathcal{V}}||\mathbf{X}_{:,v} - \tilde{\mathbf{X}}_{:,v}||^2$ is the mean squared error
between the input and output embedding matrices $\mathbf{X}, \tilde{\mathbf{X}} \in \mathbb{R}^{D \times |\mathcal{V}|}$.
We jointly optimise masks matrices  $\{ \mathbf{M}_d \in \mathbb{R}^{|\mathcal{V}| \times |\mathcal{V}|} \}_{d\in\mathcal{K}} $, computed from hidden layer parameters $\mathbf{H}$,  and the autoencoder parameters $\{ \mathbf{W}^{(0)} \in \mathbb{R}^{K \times D}, \mathbf{W}^{(1)} \in \mathbb{R}^{D \times K}, \mathbf{b}^{(0)} \in \mathbb{R}^K, \mathbf{b}^{(1)} \in \mathbb{R}^D\} $.
We determined the following regularization terms as optimal for promoting interpretable dimensions:

\begin{itemize}[leftmargin=*]
\item Induced subgraphs might have minimal overlaps between each other in order to be interpreted as communities. 
Inspired by graph clustering \cite{bianchi2020spectral}, we squeeze embedding mask matrices into one partition matrix $\mathbf{P} \in \mathbb{R}^{K\times|\mathcal{V}|}$, with entries $ P_{d,v} = \sum_{u\in\mathcal{V}} m^*_d(u,v)$, computed by aggregating edge reconstruction scores with the same target node\footnote{Aggregating over the source node would give the same result since we work with undirected graphs and $M_d(u,v)=M_d(v,u)$.}. In order to encourage relevant subgraphs to be incorporated into different embedding axes, we optimize the following \emph{Orthogonality Loss}:
\begin{equation}
    \mathcal{L}_{orth} = \mathrm{MSE}(\frac{\mathbf{P}\mathbf{P}^{\mathrm{T}}}{||\mathbf{P}\mathbf{P}^{\mathrm{T}}||_F}, \frac{\mathbf{1}_{K}}{||\mathbf{1}_{K}||_F})
\end{equation}
The use of node-based partition is due to scalability reasons: the entries of the partition matrix $P_{d,v} = \sum_{u} m^*_d(u,v) \propto h_d(\mathbf{v}) \big[\sum_{u} h_d(\mathbf{u})\big]$ can be computed avoiding the explicit calculation and caching of $\mathcal{O}(K\times|\mathcal{V}|\times|\mathcal{V}|)$ parameter for graph masks, reducing the complexity to $\mathcal{O}(K\times|\mathcal{V}|)$.
\item In order to avoid degenerate solutions due to the orthogonality constraint, e.g. all relevant subgraphs reconstructed in the same dimension, we enforce the size of every mask $s_d = \sum_{u,v} m^*_d(u,v)$ to be non-zero. This constraint is accomplished by maximizing the entropy of the size variables $\{s_d\}_{d \in \mathcal{K}}$, opportunely normalized, or equivalently minimizing the \emph{Size Loss}:
\begin{equation}
    \mathcal{L}_{size} = \log |\mathcal{K}| + \sum_{d\in\mathcal{K}} \frac{s_d}{\sum_{q \in \mathcal{K}}s_q} \log \frac{s_d}{\sum_{q \in \mathcal{K}}s_q}
\end{equation}

\end{itemize}

The full objective loss is given by:
\begin{equation}
    \mathcal{L} = \mathcal{L}_{ac}(\mathbf{X}, \tilde{\mathbf{X}}) + \mathcal{L}_{orth}(\mathbf{H}) + \mathcal{L}_{size}(\mathbf{H})
\end{equation}

\section{Experiments}
\label{sec:experiments}

In this section, we present the results of our study on the \textsc{Dine} model from different perspectives. The main objective is to address the following research questions:

\begin{description}
\item [RQ1] How does the interpretability of \textsc{Dine} compare to those of standard embedding techniques?
\item [RQ2] How well does \textsc{Dine} perform in the link prediction task?
\item [RQ3] Is \textsc{Dine} suitable for practical use, particularly in scenarios requiring scalability?
\end{description}

In the following sections, we describe the data, models, and tasks used in the comparison to address our research questions.

\subsection{Data and Models}

\begin{table}[ht]
 \caption{Summary statistics about real-world graph data. In order: number of nodes $|\mathcal{V}|$, number of edges $|\mathcal{E}|$, number of extracted communities $|\mathcal{C}|$ with Louvain method.}
    \makebox[\linewidth]{
    \centering
    \begin{tabular}{@{}l@{~~}cccc@{}}
         \toprule
         Dataset & $|\mathcal{V}|$ & $|\mathcal{E}|$  & $|\mathcal{C}|$ \\
         \midrule
         \cora & 2,485 & 5,069 &  28\\
         \citeseer & 2,110 & 3,668  & 35\\
         \pubmed  & 19,717 & 44,324  & 38\\
         \blogcatalog  & 5,196 & 171,743  & 10\\
         \flickr  & 7,575 & 239,738  & 9\\
         \wiki & 2,357 & 11,592  & 17\\
        
         \bottomrule
    \end{tabular}
    }
    \label{tab:datastats}
\end{table}

We present our results on a variety of benchmark datasets used in prior work \cite{yang2020scaling}: three citation networks (\cora, \citeseer~and \pubmed), two social networks (\blogcatalog~and \flickr), and a web pages network (\wiki).  
Despite their original format, we restrict our analysis to the large connected component of any graph, considered unweighted and undirected. 

As described in Section \ref{sec:interpret_scores}, we rely on \textit{ground-truth link partitions} for interpretability metrics based on community structure. 
Many empirical graphs have node metadata that can be used for node-to-community mapping.
However, the use of metadata as structure-aware labels has recently been criticized by previous works \cite{peel2017ground, hric2014community}.
Instead, we use community detection to discover partition labels. 
We avoid computationally expensive and overlapping community detection methods~\cite{ding2016overlapping, fortunato2010community,ahn2010link, evans2009line} and use the arguably intuitive and simpler Louvain detection method \cite{blondel2008fast} to derive edge labels based on node-level graph communities.
Specifically, 
we ran Louvain detection method \cite{blondel2008fast} to extract the node-level communities $\mathcal{C} = \{\mathcal{C}_1, \dots, \mathcal{C}_m\}$ 
and we assign paritition label for a given edge $(u,v)$ the set $\{c(u), c(v)\}$, where $c:\mathcal{V}\rightarrow \mathcal{C}$ is the node-level community membership function. We report datasets statistics in Table \ref{tab:datastats}. 

We use the following baseline methods:
\begin{itemize}[leftmargin=*]
    \item \textsc{DeepWalk} \cite{perozzi2014deepwalk}, skip-gram based model that computes node embeddings from random walks co-occurrence statistics. We train \textsc{Node2Vec}\footnote{\texttt{https://github.com/eliorc/node2vec}} for 5 epochs with the following parameters: \texttt{p=1}, \texttt{q=1}, \texttt{walk\_length=10}, \texttt{num\_walks=20}, \texttt{window\_size=5}.
    \item GAE \cite{kipf2016variational}, neural network model with a GCN encoder trained on adjacency matrix reconstruction. The model\footnote{\texttt{https://github.com/zfjsail/gae-pytorch}} is trained for 200 iterations using Adam optimizer and learning rate of 0.01 as described in the main paper. The GCN hidden layer size is taken double the output size.
    \item \textsc{GemSec} \cite{rozemberczki_gemsec_2019}, a variation of \textsc{DeepWalk} which jointly learns node embeddings and node clusters. We train the model\footnote{\texttt{https://github.com/benedekrozemberczki/karateclub}} with the same configuration as \textsc{DeepWalk}, plus the number of clusters fixed equal to the number of dimensions.
    \item \textsc{SPINE} \cite{subramanian_spine_2018}, a post-processing technique based on $k-sparse$ denoising autoencoder to generate sparse embeddings. We train the original model\footnote{\texttt{https://github.com/harsh19/SPINE}} for 2000 iterations with sparsity 0.15 and learning rate 0.1.
\end{itemize}

\textsc{Dine} is trained for $2000$ iterations,  and learning rate of $0.1$,  with \textsc{DeepWalk} and GAE embeddings, to show the capability of \textsc{Dine} in handling models with different inductive biases. With \textsc{Spine} we post-process \textsc{DeepWalk} vectors for a dedicated analysis.   

\begin{table*}[t!]
    \begin{minipage}{0.6\linewidth}
    \centering
    \begin{small}
        \captionof{table}{Community-aware scores for interpretability evaluation of different embedding methods. For each dataset, we highlight the best (highest) score.}
        \label{tab:int_community}
        \begin{tabular}{lccccc}
         \toprule
          & \citeseer & \pubmed & \blogcatalog & \flickr & \wiki \\
         \midrule
          \textsc{DeepWalk} & 
          \makecell{0.433\\{\smaller($\pm$0.014)}}& 
          \makecell{0.422\\{\smaller($\pm$0.011)}}&
          \makecell{0.432\\{\smaller($\pm$0.007)}}&
          \makecell{0.499\\{\smaller($\pm$0.105)}}&
          \makecell{0.445\\{\smaller($\pm$0.025)}}\\
          
          GAE & 
          \makecell{0.420\\{\smaller($\pm$0.010)}}&
          \makecell{0.487\\{\smaller($\pm$0.019)}}&
          \makecell{0.496\\{\smaller($\pm$0.006)}}&
          \makecell{0.625\\{\smaller($\pm$0.021)}}&
          \makecell{0.460\\{\smaller($\pm$0.014)}}\\
          
         \textsc{GemSec}  &
         \makecell{0.446\\{\smaller($\pm$0.014)}}& 
         \makecell{0.431\\{\smaller($\pm$0.009)}}&
         \makecell{0.392\\{\smaller($\pm$0.030)}}&
         \makecell{0.421\\{\smaller($\pm$0.022)}}&
         \makecell{0.454\\{\smaller($\pm$0.002)}}\\
         
          \textsc{DeepWalk}+\textsc{DINE}  & 
          \makecell{\textbf{0.641}\\{\smaller($\pm$0.027)}}& 
          \makecell{\textbf{0.605}\\{\smaller($\pm$0.027)}}&
          \makecell{0.590\\{\smaller($\pm$0.023)}}&
          \makecell{\textbf{0.657}\\{\smaller($\pm$0.014)}}&
          \makecell{\textbf{0.652}\\{\smaller($\pm$0.010)}}\\
          
          GAE+\textsc{DINE}  & 
          \makecell{0.526\\{\smaller($\pm$0.008)}}& 
          \makecell{0.593\\{\smaller($\pm$0.011)}}&
          \makecell{\textbf{0.600}\\{\smaller($\pm$0.008)}}&
          \makecell{0.595\\{\smaller($\pm$0.013)}}&
          \makecell{0.631\\{\smaller($\pm$0.017)}}\\
          
         \midrule
         \textsc{DeepWalk} & 
         \makecell{0.491\\{\smaller($\pm$0.016)}}&
         \makecell{0.457\\{\smaller($\pm$0.006)}}&
         \makecell{0.483\\{\smaller($\pm$0.009)}}&
         \makecell{0.469\\{\smaller($\pm$0.018)}}&
         \makecell{0.496\\{\smaller($\pm$0.009)}}\\
         
         \textsc{DeepWalk}+\textsc{SPINE}& 
         \makecell{0.543\\{\smaller($\pm$0.060)}}& 
         \makecell{\textbf{0.608}\\{\smaller($\pm$0.037)}}&
         \makecell{\textbf{0.632}\\{\smaller($\pm$0.082)}}&
         \makecell{\textbf{0.703}\\{\smaller($\pm$0.001)}}&
         \makecell{0.611\\{\smaller($\pm$0.015)}}\\
         
         \textsc{DeepWalk}+\textsc{DINE}& 
         \makecell{\textbf{0.641}\\{\smaller($\pm$0.027)}}& 
         \makecell{0.605\\{\smaller($\pm$0.027)}}&
         \makecell{0.590\\{\smaller($\pm$0.023)}}&
         \makecell{0.657\\{\smaller($\pm$0.014)}}&
         \makecell{\textbf{0.652}\\{\smaller($\pm$0.010)}}\\
        
         \bottomrule
    \end{tabular}
    \end{small}
    \end{minipage}\hfill
     \begin{minipage}{0.4\linewidth}
     \begin{subfigure}{0.49\linewidth}
         \centering
         \includegraphics[width=\linewidth]{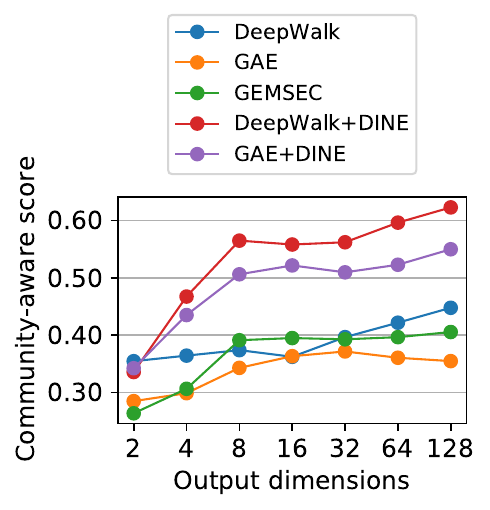}
     \end{subfigure}
     \begin{subfigure}{0.475\linewidth}
         \centering
         \includegraphics[width=\linewidth]{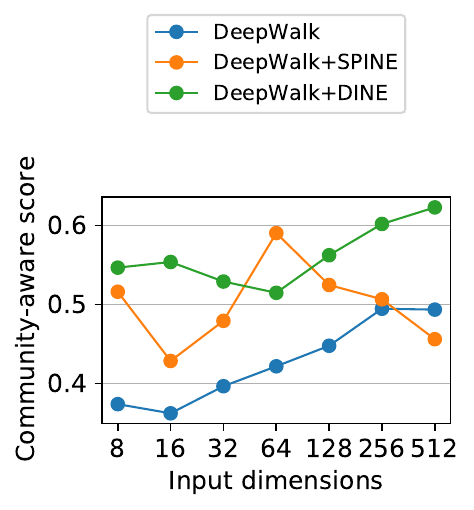}
     \end{subfigure}
    \captionof{figure}{Community-aware scores for \textsc{Cora} (higher is better). 
    On the left, we compare the community-aware scores of different dense embedding methods when varying the number of output dimensions and choosing the best score among models over different number of input dimensions.
    On the right, we compare the scores of \textsc{DeepWalk} and \textsc{Spine} when varying the number of input dimensions and choosing the best score among models over different number of output dimensions.
     }
     \label{fig:cora_community}
    \end{minipage}
    
\end{table*}
\begin{table*}[ht]
    \begin{minipage}{0.6\linewidth}
    \centering
    \begin{small}
        \captionof{table}{Sparsity-aware scores for interpretability evaluation of different embedding methods. For each dataset, we highlight the best (lowest) score.}        
        \label{tab:int_sparsity}
        \begin{tabular}{lccccc}
         \toprule
          & \citeseer & \pubmed & \blogcatalog & \flickr & \wiki \\
         \midrule
          \textsc{DeepWalk} & 
          \makecell{0.778\\{\smaller($\pm$0.006)}}& 
          \makecell{0.849\\{\smaller($\pm$0.001)}}&
          \makecell{0.865\\{\smaller($\pm$0.002)}}&
          \makecell{0.874\\{\smaller($\pm$0.004)}}&
          \makecell{0.820\\{\smaller($\pm$0.003)}}\\
          
          GAE & 
          \makecell{0.825\\{\smaller($\pm$0.004)}}&
          \makecell{0.851\\{\smaller($\pm$0.005)}}&
          \makecell{0.871\\{\smaller($\pm$0.009)}}&
          \makecell{0.785\\{\smaller($\pm$0.001)}}&
          \makecell{0.844\\{\smaller($\pm$0.002)}}\\
          
         \textsc{GemSec}  &  
         \makecell{0.817\\{\smaller($\pm$0.003)}}& 
         \makecell{0.865\\{\smaller($\pm$0.002)}}&
         \makecell{0.911\\{\smaller($\pm$0.003)}}&
         \makecell{0.916\\{\smaller($\pm$0.002)}}&
         \makecell{0.846\\{\smaller($\pm$0.004)}}\\
         
          \textsc{DeepWalk}+\textsc{DINE}  & 
          \makecell{\textbf{0.630}\\{\smaller($\pm$0.015)}}& 
          \makecell{\textbf{0.684}\\{\smaller($\pm$0.242)}}&
          \makecell{\textbf{0.749}\\{\smaller($\pm$0.006)}}&
          \makecell{\textbf{0.372}\\{\smaller($\pm$0.069)}}&
          \makecell{\textbf{0.688}\\{\smaller($\pm$0.003)}}\\
          
          GAE+\textsc{DINE}  & 
          \makecell{0.728\\{\smaller($\pm$0.009)}}& 
          \makecell{0.805\\{\smaller($\pm$0.003)}}&
          \makecell{0.820\\{\smaller($\pm$0.012)}}&
          \makecell{0.711\\{\smaller($\pm$0.021)}}&
          \makecell{0.756\\{\smaller($\pm$0.003)}}\\
          
         \midrule
         \textsc{DeepWalk} & 
         \makecell{0.755\\{\smaller($\pm$0.003)}}&
         \makecell{0.836\\{\smaller($\pm$0.001)}}&
         \makecell{0.840\\{\smaller($\pm$0.001)}}&
         \makecell{0.874\\{\smaller($\pm$0.004)}}&
         \makecell{0.796\\{\smaller($\pm$0.004)}}
         \\
         \textsc{DeepWalk}+\textsc{SPINE}& 

         \makecell{0.703\\{\smaller($\pm$0.025)}}& 
         \makecell{\textbf{0.741}\\{\smaller($\pm$0.033)}}&
         \makecell{\textbf{0.610}\\{\smaller($\pm$0.089)}}&
         \makecell{0.729\\{\smaller($\pm$0.125)}}&
         \makecell{\textbf{0.529}\\{\smaller($\pm$0.097)}}
         \\
         \textsc{DeepWalk}+\textsc{DINE}& 
         \makecell{\textbf{0.630}\\{\smaller($\pm$0.015)}}& 
         \makecell{0.748\\{\smaller($\pm$0.008)}}&
         \makecell{0.749\\{\smaller($\pm$0.006)}}&
         \makecell{\textbf{0.681}\\{\smaller($\pm$0.060)}}&
         \makecell{0.688\\{\smaller($\pm$0.003)}}
         \\
         \bottomrule
    \end{tabular}
    \end{small}
    \end{minipage}\hfill
     \begin{minipage}{0.4\linewidth}
     \begin{subfigure}{0.49\linewidth}
         \centering
         \includegraphics[width=\linewidth]{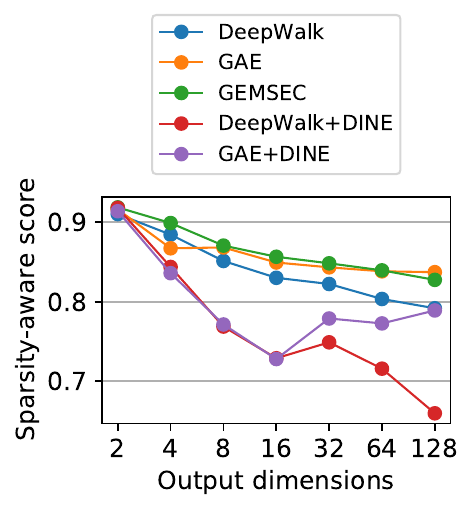}
     \end{subfigure}
     \begin{subfigure}{0.475\linewidth}
         \centering
         \includegraphics[width=\linewidth]{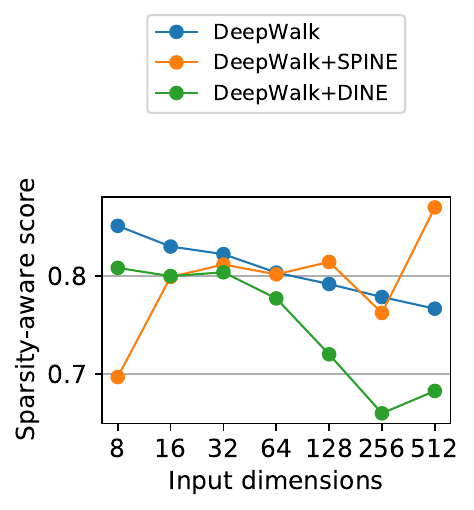}
     \end{subfigure}
     \captionof{figure}{Sparsity-aware scores for  \textsc{Cora} (lower is better). 
    On the left, we compare the sparsity-aware scores of different dense embedding methods when varying the number of output dimensions and choosing the best score among models with a different number of input dimensions.
    On the right, we compare the scores of \textsc{DeepWalk} and \textsc{Spine} when varying the number of input dimensions and choosing the best score among models with a different number of output dimensions.
     }
         \label{fig:cora_sparsity}
    \end{minipage}
\end{table*}
\begin{table*}[ht]
    \begin{minipage}{0.6\linewidth}
    \centering
    \begin{small}
        \captionof{table}{ROC-AUC scores for link prediction performance of different embedding methods. For each dataset, we show in boldface letters the best score, and we highlight with gray background other methods that reach at least $95\%$ of the best model score.}   
        \label{tab:linkpred}
        \begin{tabular}{lccccc}
         \toprule
          & \citeseer & \pubmed & \blogcatalog & \flickr & \wiki \\
         \midrule
          \textsc{DeepWalk} & 
          \makecell{\colorbox{lightgray}{0.945}\\{\smaller($\pm$0.007)}}& 
          \makecell{\textbf{0.964}\\{\smaller($\pm$0.002)}}&
          \makecell{0.756\\{\smaller($\pm$0.001)}}&
          \makecell{0.656\\{\smaller($\pm$0.001)}}&
          \makecell{0.848\\{\smaller($\pm$0.003)}}\\
          
          GAE & 
          \makecell{\textbf{0.952}\\{\smaller($\pm$0.006)}}&
          \makecell{\colorbox{lightgray}{0.958}\\{\smaller($\pm$0.002)}}&
          \makecell{\textbf{0.869}\\{\smaller($\pm$0.003)}}&
          \makecell{\colorbox{lightgray}{0.894}\\{\smaller($\pm$0.004)}}&
          \makecell{\textbf{0.942}\\{\smaller($\pm$0.002)}}\\
          
         \textsc{GemSec}  &  
         \makecell{\colorbox{lightgray}{0.942}\\{\smaller($\pm$0.007)}}& 
         \makecell{\colorbox{lightgray}{0.951}\\{\smaller($\pm$0.003)}}&
         \makecell{0.794\\{\smaller($\pm$0.002)}}&
         \makecell{0.670\\{\smaller($\pm$0.003)}}&
         \makecell{\colorbox{lightgray}{0.905}\\{\smaller($\pm$0.004)}}\\
         
          \textsc{DeepWalk}+\textsc{DINE}  & 
          \makecell{\colorbox{lightgray}{0.947}\\{\smaller($\pm$0.010)}}& 
          \makecell{\colorbox{lightgray}{0.933}\\{\smaller($\pm$0.012)}}&
          \makecell{0.747\\{\smaller($\pm$0.006)}}&
          \makecell{0.759\\{\smaller($\pm$0.009)}}&
          \makecell{\colorbox{lightgray}{0.905}\\{\smaller($\pm$0.004)}}\\
          
          GAE+\textsc{DINE}  & 
          \makecell{\colorbox{lightgray}{0.920}\\{\smaller($\pm$0.009)}}& 
          \makecell{\colorbox{lightgray}{0.961}\\{\smaller($\pm$0.001)}}&
          \makecell{0.820\\{\smaller($\pm$0.010)}}&
          \makecell{\textbf{0.911}\\{\smaller($\pm$0.003)}}&
          \makecell{\textbf{0.942}\\{\smaller($\pm$0.004)}}\\
          
         \midrule
         \textsc{DeepWalk} & 
         \makecell{\textbf{0.949}\\{\smaller($\pm$0.008)}}&
         \makecell{\textbf{0.965}\\{\smaller($\pm$0.001)}}&
         \makecell{\textbf{0.758}\\{\smaller($\pm$0.001)}}&
         \makecell{0.667\\{\smaller($\pm$0.002)}}&
         \makecell{0.848\\{\smaller($\pm$0.002)}}\\
         
         \textsc{DeepWalk}+\textsc{SPINE}& 
         \makecell{0.897\\{\smaller($\pm$0.021)}}& 
         \makecell{0.902\\{\smaller($\pm$0.006)}}&
         \makecell{0.673\\{\smaller($\pm$0.014)}}&
         \makecell{0.632\\{\smaller($\pm$0.035)}}&
         \makecell{0.816\\{\smaller($\pm$0.014)}}\\
         
         \textsc{DeepWalk}+\textsc{DINE}& 
         \makecell{\colorbox{lightgray}{0.947}\\{\smaller($\pm$0.010)}}& 
         \makecell{\colorbox{lightgray}{0.933}\\{\smaller($\pm$0.012)}}&
         \makecell{\colorbox{lightgray}{0.747}\\{\smaller($\pm$0.006)}}&
         \makecell{\textbf{0.759}\\{\smaller($\pm$0.009)}}&
         \makecell{\textbf{0.905}\\{\smaller($\pm$0.004)}}\\
        
         \bottomrule
    \end{tabular}
    \end{small}
    \end{minipage}\hfill
     \begin{minipage}{0.4\linewidth}
     \begin{subfigure}{0.49\linewidth}
         \centering
         \includegraphics[width=\linewidth]{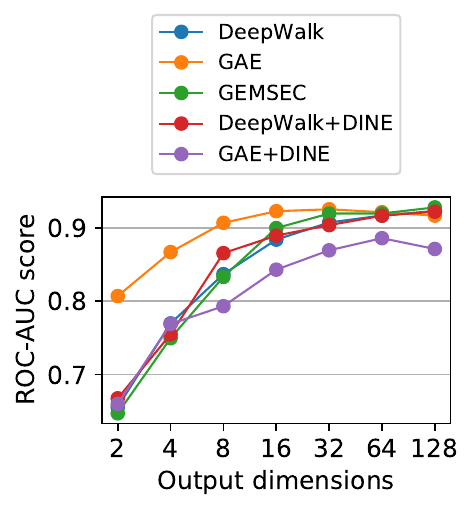}
     \end{subfigure}
     \begin{subfigure}{0.475\linewidth}
         \centering
         \includegraphics[width=\linewidth]{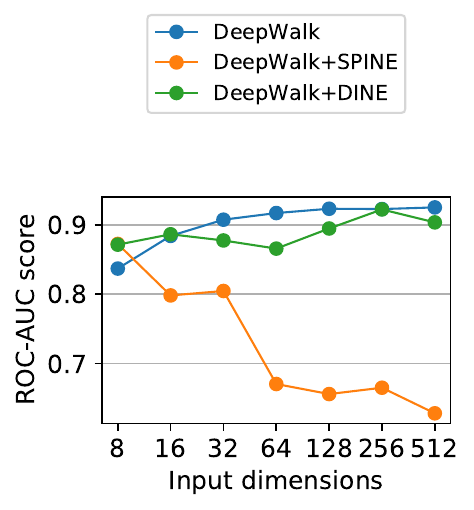}
     \end{subfigure}
     \captionof{figure}{ROC-AUC scores for link prediction in \textsc{Cora}. 
     On the left, we compare the ROC-AUC scores of different dense embedding methods when varying the number of output dimensions and choosing the best score among models with a different number of input dimensions.
    On the right, we compare the scores of \textsc{DeepWalk} and \textsc{Spine} when varying the number of input dimensions and choosing the best score among models with a different number of output dimensions.}
     \label{fig:cora_linkpred}
    \end{minipage}
\end{table*}
\subsection{Tasks Description}

For answering \textbf{RQ1}, we measure the interpretability of \textsc{Dine} in comparison to our baseline methods. To do so, we compute for any embedding dimension interpretability scores that we have defined in \S \ref{sec:interpret_scores}. 
Instead of averaging over all dimensions for comparing the models, we focus on a subset of effective dimensions $\mathcal{D}_{eff}$ that encode the majority of edge information,  to avoid potential noise from the less important dimensions.   
Specifically, after computing $I_d^{(com)}$ or $I_d^{(sp)}$, we select the top-ranked dimensions that cumulatively contribute to the reconstruction of at least 90\% of the graph edges, i.e., $|\bigcup_{d \in \mathcal{D}_{eff}} \mathcal{E}_d| = 90\% |\mathcal{E}|$. Thus we compute global scores as $I^{(com|sp)}_{eff} = \frac{1}{|\mathcal{D}_{eff}|}\sum_{d \in \mathcal{D}_{eff}}I^{(com|sp)}_d$.
For answering \textbf{RQ2}, we measure the link prediction performance of \textsc{Dine} in comparison to our baseline methods.
To do so, before training every method, we randomly remove $10\%$ of the edges that are used as positive examples for the link prediction task on one hand. 
On the other hand, we also sample the same number of node pairs from the set of non-existing links as negative examples. 
The task consists in ranking the collected node pairs with the scoring function $\Delta$ and evaluating the classification performance with the ROC-AUC score.
\subsection{Results}
\begin{figure*}[h]
        \centering
        \includegraphics[width=0.8\linewidth]{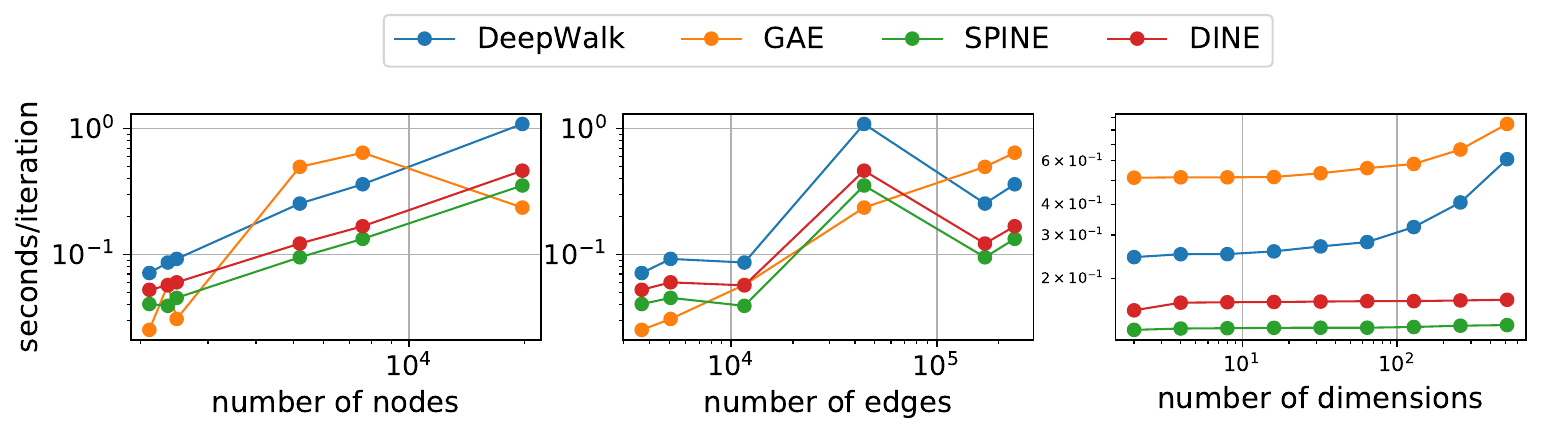}
        \caption{Normalized execution times for  different embedding methods when trained on multiple datasets. In the left, we compare running times while varying the number of nodes with 128-dim output embeddings; in the center, we compare running times while varying the number of edges with 128-dim output embeddings; in the right, we compare running times while varying the number of embedding dimensions in \flickr~dataset. }
     \label{fig:scalability}
\end{figure*}

In our experiments, all methods, with the exception of \textsc{Spine}, are trained to produce embedding vectors with dimensions in the set $\{2, 4, 8, 16, 32, 64, 128\}$, referred to as \textit{output dimensions}. 
On the other hand, the \textit{input dimensions} for \textsc{DeepWalk} and GAE vectors used for training \textsc{Dine} and \textsc{Spine} (only \textsc{DeepWalk} for the latter), taken from the set $\{8, 16, 32, 64, 128, 256, 512\}$, define the \textit{dimensionality} of the input vectors.
For \textsc{Spine}, due to the presence of the overcomplete layer, we chose output dimensions to be multiples of the input dimensions, i.e. between $\times 1$ and $\times 8$.

In our comparison, we evaluated \textsc{Dine} against both dense and sparse methods. For the comparison with dense methods, \textsc{Dine} was trained using \textsc{DeepWalk} and GAE vectors, with performance reported across different output dimensions. To compare with sparse methods, both \textsc{Spine} and \textsc{Dine} were trained using \textsc{DeepWalk} vectors, with their performance reported across different input \textsc{DeepWalk} dimensions. The results for dense and sparse embeddings are reported respectively at the top and at the bottom of each table, with average evaluation score and standard deviation computed over 5 separate training runs. Additional figures are reported in the Appendix with all the results not shown in the main paper.

\noindent\textbf{Interpretability (RQ1).}  
Best scores for our interpretability metric $I^{(com)}_{eff}$ and $I^{(sp)}_{eff}$  are reported in Tables \ref{tab:int_community} and \ref{tab:int_sparsity} respectively for all the datasets, with detailed plots in Figures \ref{fig:cora_community} and \ref{fig:cora_sparsity} for \textsc{Cora}. 
We notice that the combination \textsc{DeepWalk}+\textsc{Dine} performs well with respect to dense embeddings in almost every dataset where their vector dimensions are sparse and are grounded in the community structure. 
The model \textsc{GAE}+\textsc{Dine} is less interpretable than \textsc{DeepWalk}+\textsc{Dine}, but still more interpretable than the other dense baselines. 
Both \textsc{Dine} and \textsc{Spine} models trained on \textsc{DeepWalk} vectors demonstrate good interpretability, obtaining the best scores in half of the datasets each. Our results confirm the well-known property of vector sparsity improving the interpretability of representations \cite{du2019techniques}. Additionally, we observe that both interpretability metrics improve as the embedding dimensions increase in the \textsc{Cora} dataset (Figures \ref{fig:cora_community} and \ref{fig:cora_sparsity}). This is true also in the other datasets shown in the Appendix, providing important guidance for choosing the appropriate embedding size in real-world applications.

\noindent \textbf{Link prediction (RQ2).} ROC-AUC scores are documented in Table \ref{tab:linkpred} for all datasets, with a detailed illustration for \textsc{Cora} in Figure \ref{fig:cora_linkpred}. The results show that in citation networks, all models, including \textsc{Dine} retrofitted embeddings, perform similarly to the optimal results obtained from dense embeddings. 
 The implication of this result is that we do not have to trade task performance with increased interpretability. 
 In other datasets, the best scores are obtained by GAE and GAE+\textsc{Dine}. 
In fact, when comparing sparse embeddings, \textsc{DeepWalk}+\textsc{Dine} \textbf{consistently outperforms} \textsc{DeepWalk}+\textsc{Spine}, with comparable or even superior results (in the case of \flickr~ and \wiki) to \textsc{DeepWalk}. Interestingly, our results in Figure \ref{fig:cora_linkpred} also demonstrate that \textsc{Spine}'s performance decreases with increasing input dimensions, unlike the other methods.
 
\textbf{Scalability (RQ3).} 
The training times for various methods are presented in Figure~\ref{fig:scalability}, with the intervals normalized relative to the number of iterations/epochs. 
For \textsc{DeepWalk}, the intervals are further divided with respect to the \texttt{num\_walks} parameter to remove the dependence on the number of walks per node. 
Experimental results on the scalability suggest that it is possible to increase the interpretability of node representations without requiring significant computational costs.
The left panel shows that the runtimes for \textsc{DeepWalk}, \textsc{Spine}, and \textsc{Dine} increase with the number of nodes, while the center panel demonstrates that the execution time for GAE increases with the number of edges. Additionally, \textsc{Dine} has slightly longer training times compared to \textsc{Spine}, but both are faster than \textsc{DeepWalk}. The right panel indicates that the training time for \textsc{Spine} and \textsc{Dine} has a weak dependence on the number of embedding dimensions, while this dependence is more pronounced in GAE and \textsc{DeepWalk}. 

\section{ Conclusion and future work}
\label{sec:conclusions}

In this work, we presented a framework for constructing global explanations for node embeddings. We explain each embedding dimension using the important substructures of the input graph. To construct these explanations we developed a new \textit{model-agnostic} utility measure 
which computes the contributions of each dimension to predict the graph structure. Our explanations follow the desired properties of decomposability, comprehensibility and sparsity. 

With the goal of maximizing these properties, we proposed and developed \textsc{Dine}, an auto-encoder framework to enhance the interpretability of existing node embeddings. In short, \textsc{Dine} captures the structural properties encoded in an input embedding and optimizes a set of graph masks in order to promote orthogonality and sparsity of predicted sub-structures.
Our comprehensive experimental study supports our claims that \textsc{Dine} improves embedding comprehensibility over  standard node embedding techniques, without compromising the task performance. \textsc{Dine} is also preferable to the sparse method \textsc{Spine} due to its better achievements in link prediction. \textsc{Dine} scales well with respect to the input graph size, being suitable to be used in graphs with high edge density. 
Since the computation of the exact utility measure has exponential complexity as is usually the case for Shapley-based measures, the presented utility measure shares limitations common to other approximation strategies suggested in the literature. In particular, the approximation deteriorates under high interdependence among embedding features \cite{aas2021explaining}. Nevertheless, the encouraging results from our experiments support the effectiveness of this approach.

These contributions open multiple avenues for future work. Specifically, our approach can be extended to constructing interpretable node embeddings whose dimensions are aware of multi-scale subgraph structures \cite{perozzi2017don} inherent in many real-world graphs \cite{clauset2008hierarchical}. \textsc{Dine} can also be used as a plug-in architecture to facilitate interpretable learning in various graph neural network encoders \cite{kipf2016variational}.

\section*{Acknowledgments}
S. Piaggesi acknowledges partial support from the European Community program under the funding schemes: G.A. n.871042, “SoBigData++: European Integrated Infrastructure for Social Mining and Big Data Analytics”; G.A. 834756 “XAI: Science and technology for the eXplanation of AI decision making”. The funder had no role in study design, data collection and analysis, decision to publish, or preparation of the manuscript.

\bibliographystyle{ieeetr}
\bibliography{biblio}

\newpage
\begin{IEEEbiography}[{\includegraphics[width=1in,height=1.25in,trim={11in 3in 8in 5in},clip,keepaspectratio]{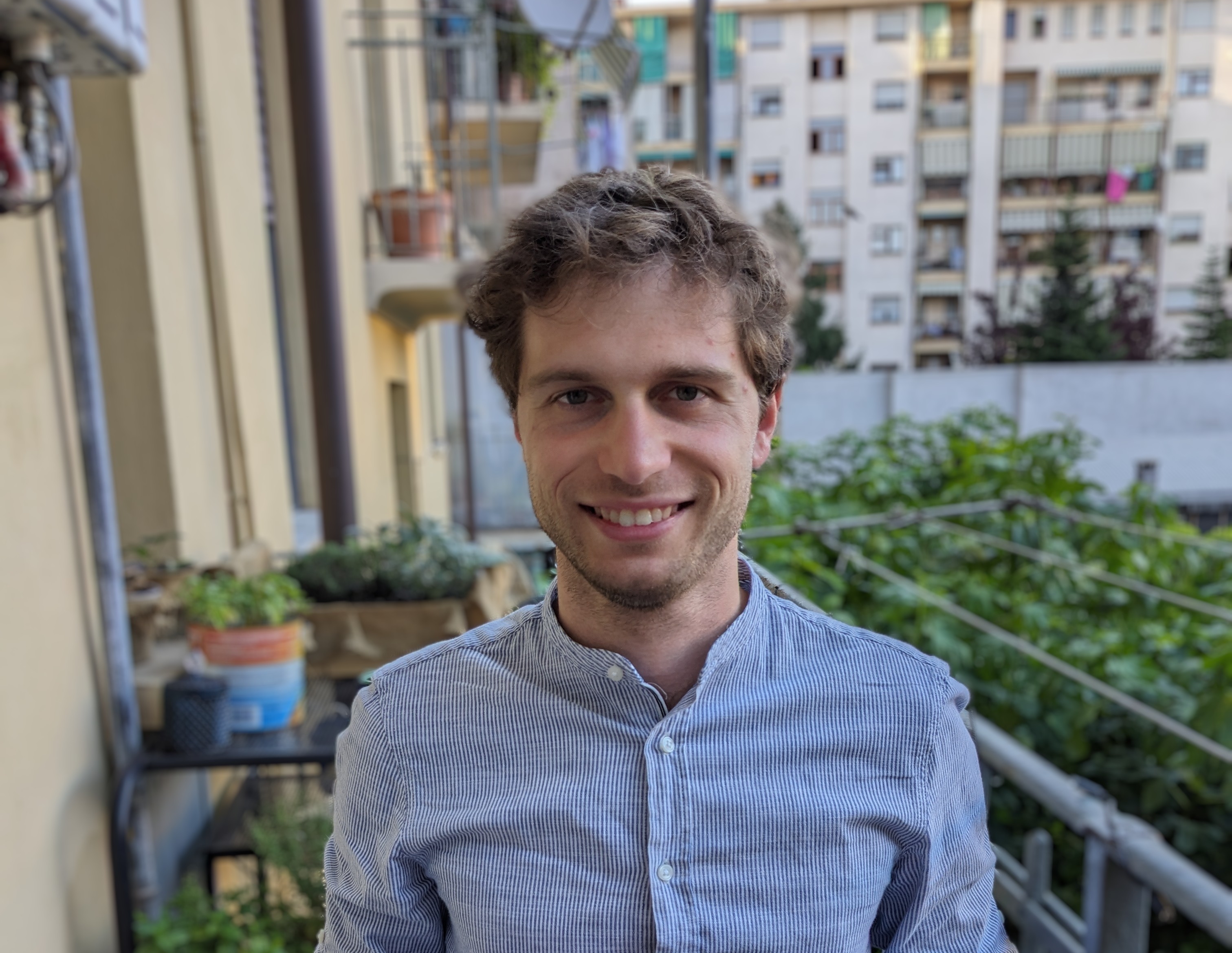}}]{Simone Piaggesi} received his Ph.D. degree in Data Science and Computation from the University of Bologna, Italy, in 2023. He is currently a Post-Doctoral Fellow in the Computer Science department of the University of Pisa, Italy. His research interests include interpretability of graph learning models. 
\end{IEEEbiography}

\begin{IEEEbiography}[{\includegraphics[width=1in,height=1.25in,clip,keepaspectratio]{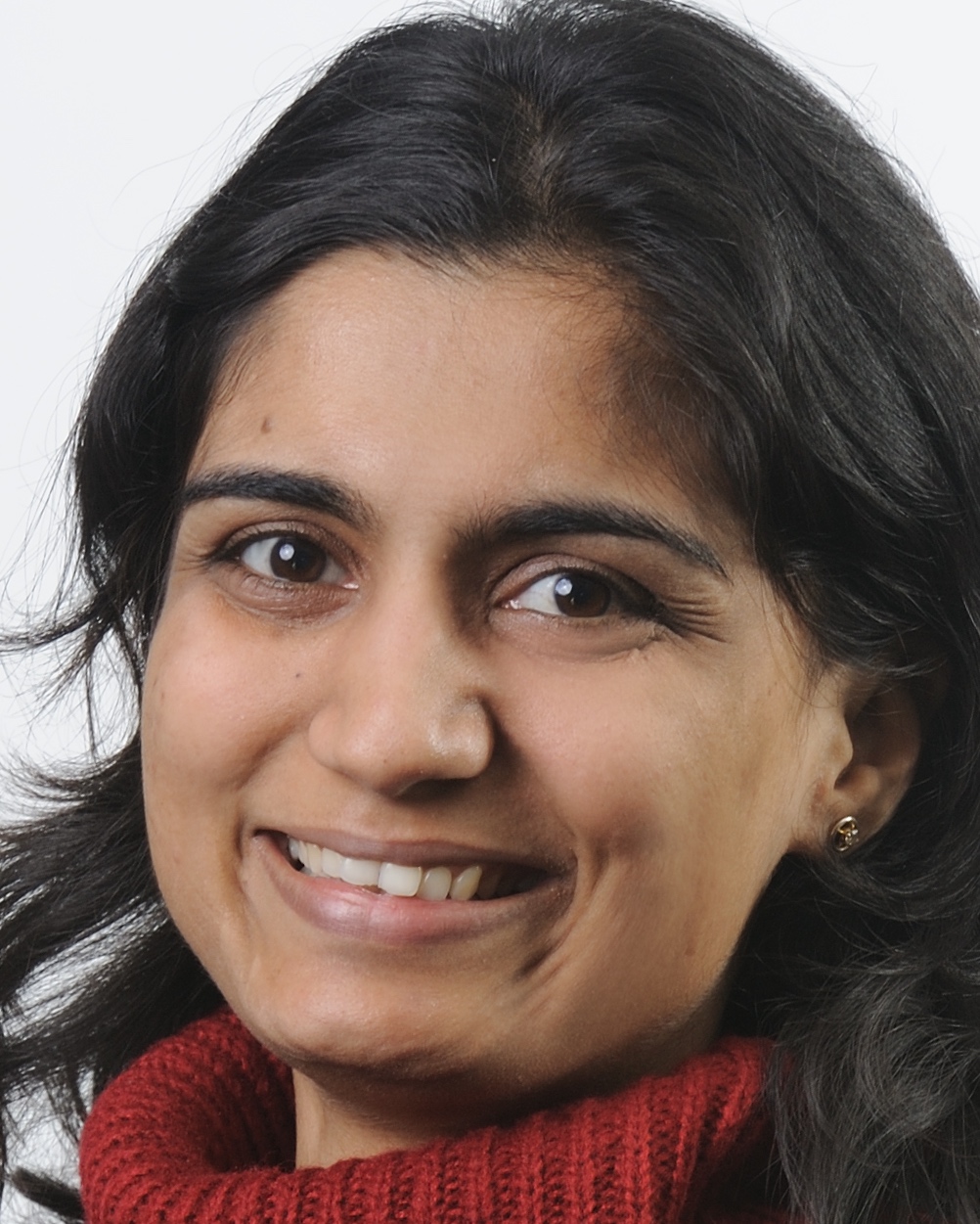}}] {Megha Khosla} is  an Assistant Professor in the Intelligent Systems department at TU Delft, Netherlands. Her main research area is Machine Learning on Graphs with focus on three key aspects of effectiveness, interpretability and privacy-preserving learning.
\end{IEEEbiography}

\begin{IEEEbiography}[{\includegraphics[width=1in,height=1.25in,trim={1.in 0.2in 0.5in 0.2in},clip,keepaspectratio]{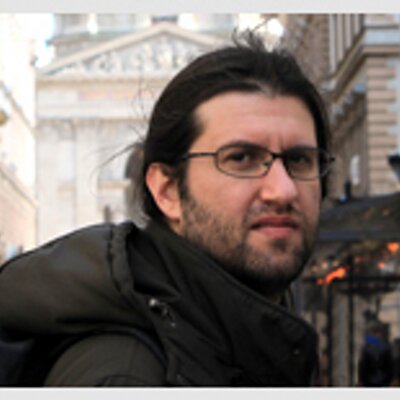}}]{Andr\'{e} Panisson} is a Principal Researcher in the CENTAI Institute in Turin,
Italy. His current research focuses on the development of tools to facilitate the explainability, fairness, and transparency of Artificial Intelligence systems. His past and current research also focuses on the intersection of Machine Learning, Network Science, and Data
Science, primarily on developing methods for the analysis, modeling, and simulation of complex phenomena in systems that involve technological and social factors.
\end{IEEEbiography}

\begin{IEEEbiography}[{\includegraphics[width=1in,height=1.25in,clip,keepaspectratio]{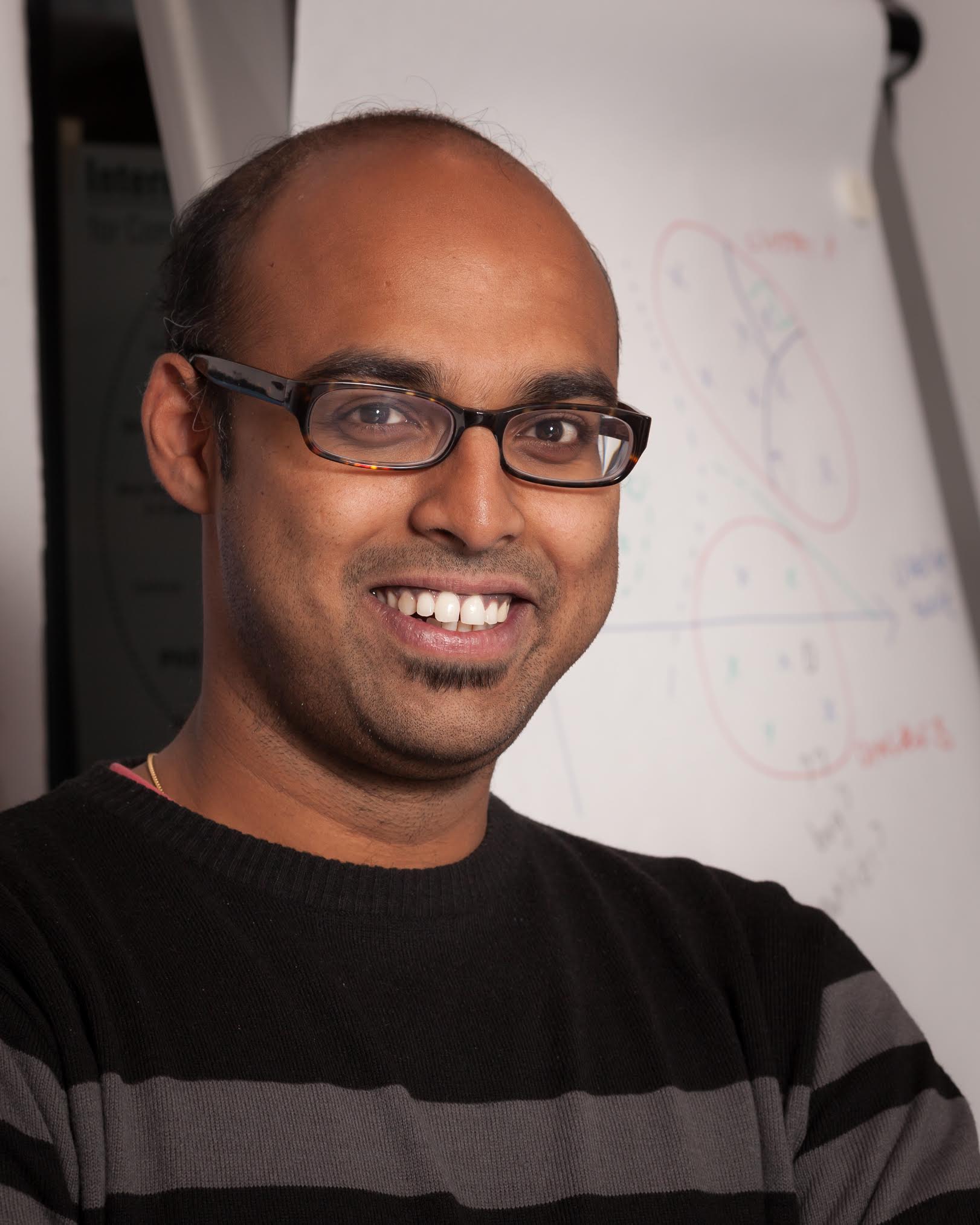}}]{Avishek Anand} is an associate professor in the Web Information Systems (WIS) at the Software Technology (ST) department at Delft University of Technology (TU Delft). He is also a member of the L3S Research Center, Hannover, Germany. One of his main research focus is interpretability of machine learning models with focus on representations from discrete input like text and graphs.
\end{IEEEbiography}

\onecolumn

\appendix

\subsection*{Extensive Results}

In Figures \ref{fig:com_int}, \ref{fig:ent_int} and \ref{fig:link} we report the whole results on multiple datasets for our experiments, that were only shown for \textsc{Cora} in the main paper (Figures 4-6). 

\setcounter{table}{0}
\renewcommand{\thetable}{A\arabic{table}}

\setcounter{figure}{0}
\renewcommand{\thefigure}{A\arabic{figure}}

\begin{figure*}[h!]
     \centering
         \centering
         \includegraphics[width=\textwidth]{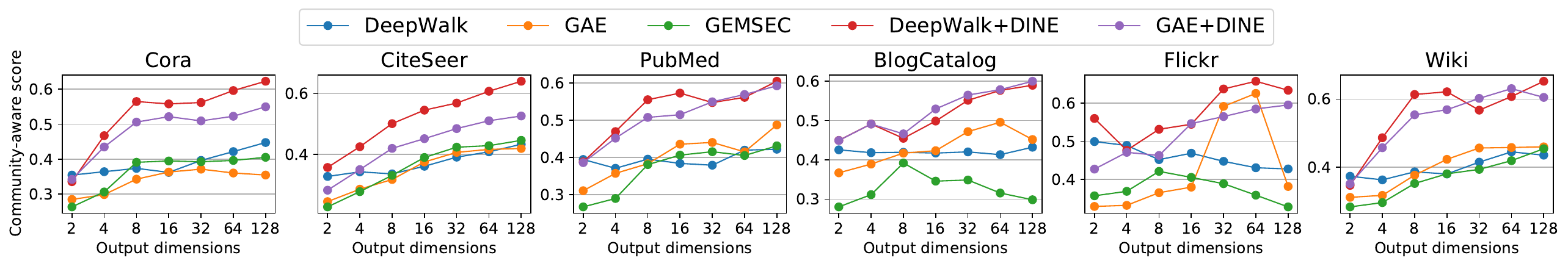}
         \centering
         \includegraphics[width=\textwidth]{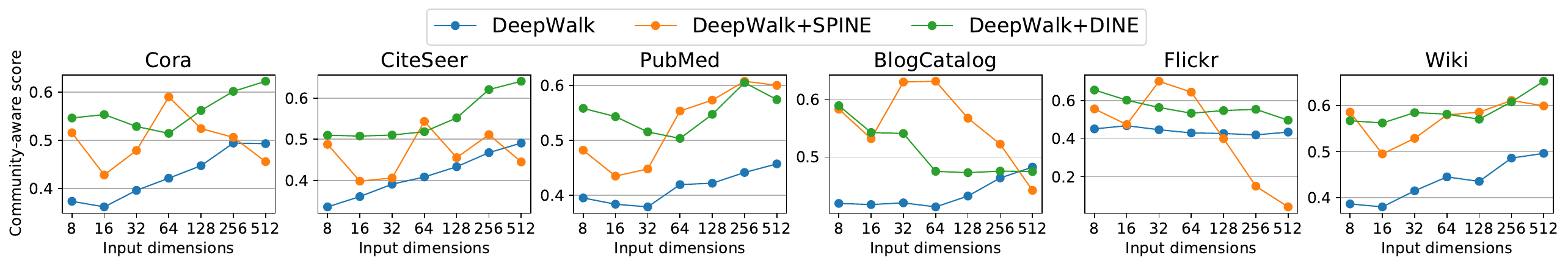}
        \caption{Community-aware scores for the interpretations of embeddings trained with different methods. 
        In the top panel, we compare the community-aware scores of different dense embedding methods when varying the number of output dimensions and choosing the best score among models with different number of input dimensions;
        in the bottom panel, we compare the scores of \textsc{DeepWalk} and \textsc{Spine} when varying the number of input dimensions and choosing the best score among models with a different number of output dimensions.}
        \label{fig:com_int}
\end{figure*}

\begin{figure*}
     \centering
         \centering
         \includegraphics[width=\textwidth]{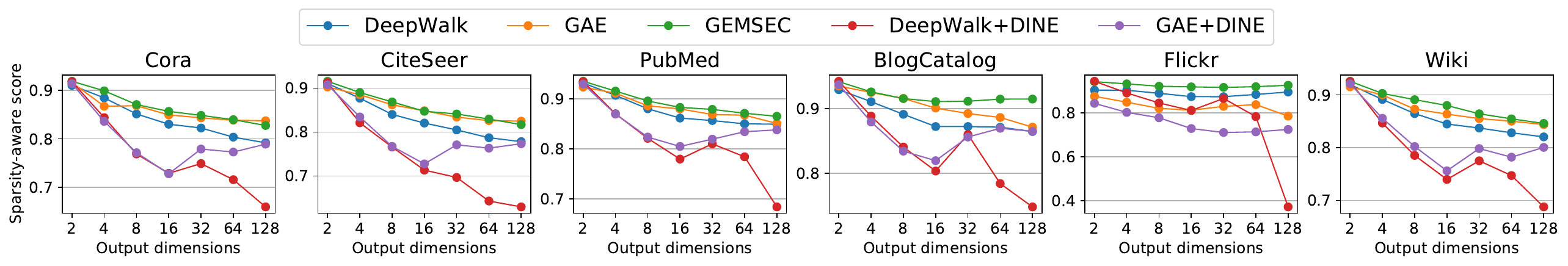}
         \centering
         \includegraphics[width=\textwidth]{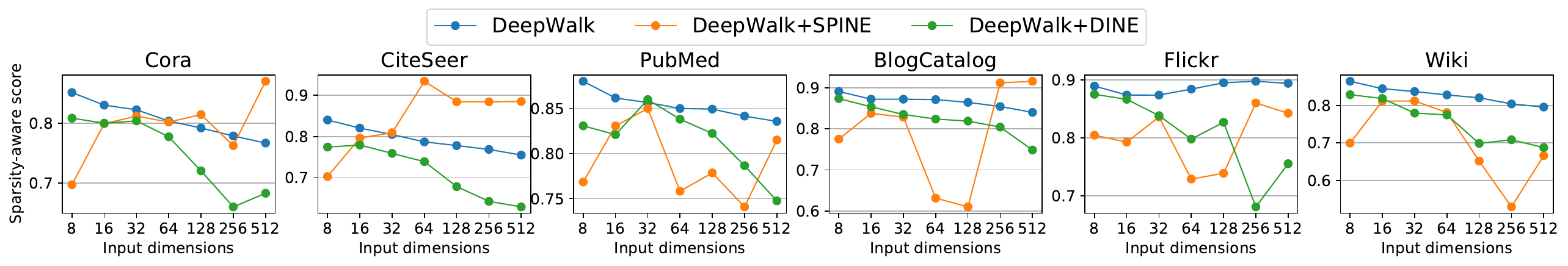}
     
    \caption{Sparsity-aware scores for the interpretations of embeddings trained with different methods. 
        In the top panel, we compare the sparsity-aware scores of different dense embedding methods when varying the number of output dimensions and choosing the best score among models with different number of input dimensions;
        in the bottom panel, we compare the scores of \textsc{DeepWalk} and \textsc{Spine} when varying the number of input dimensions and choosing the best score among models with a different number of output dimensions.}
        \label{fig:ent_int}
\end{figure*}

\begin{figure*}
     \centering
         \centering
         \includegraphics[width=\textwidth]{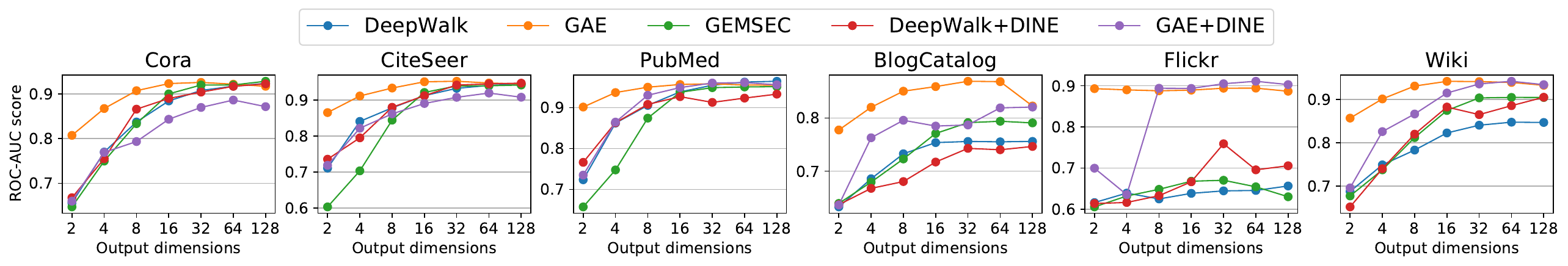}
         \centering
         \includegraphics[width=\textwidth]{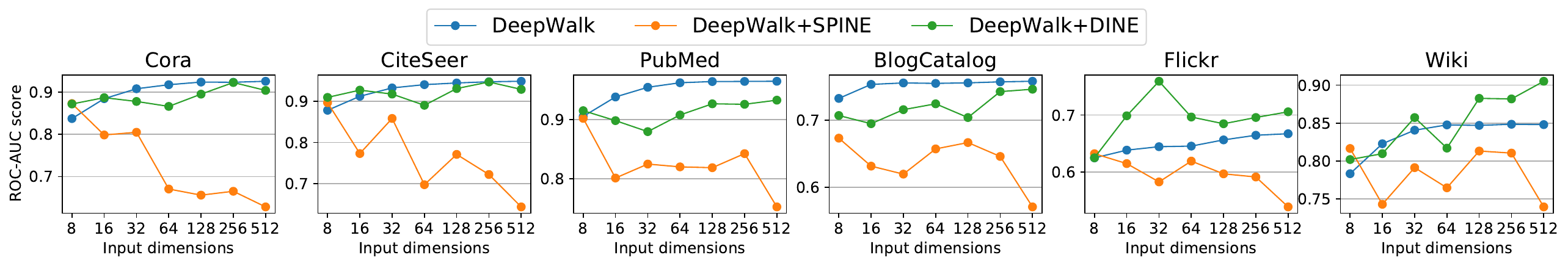}
     
    \caption{ROC-AUC scores in link prediction tasks of embeddings trained with different methods. 
     In the top panel, we compare the ROC-AUC scores of different dense embedding methods when varying the number of output dimensions and choosing the best score among models with a different number of input dimensions;
     in the bottom panel, we compare the scores of \textsc{DeepWalk} and \textsc{Spine} when varying the number of input dimensions and choosing the best score among models with a different number of output dimensions.}
        \label{fig:link}
\end{figure*}

\newpage

\subsection*{Ablation Studies}

In Figure \ref{fig:abl} we show the interpretability metrics on multiple datasets when removing regularizers from the DINE objective.

\begin{figure*}[h!]
     \centering
         \centering
         \includegraphics[width=\textwidth]{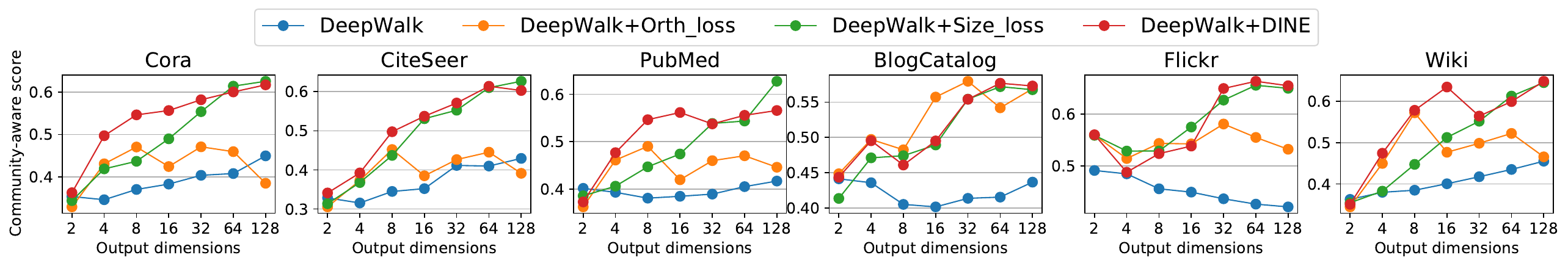}
         \centering
         \includegraphics[width=\textwidth]{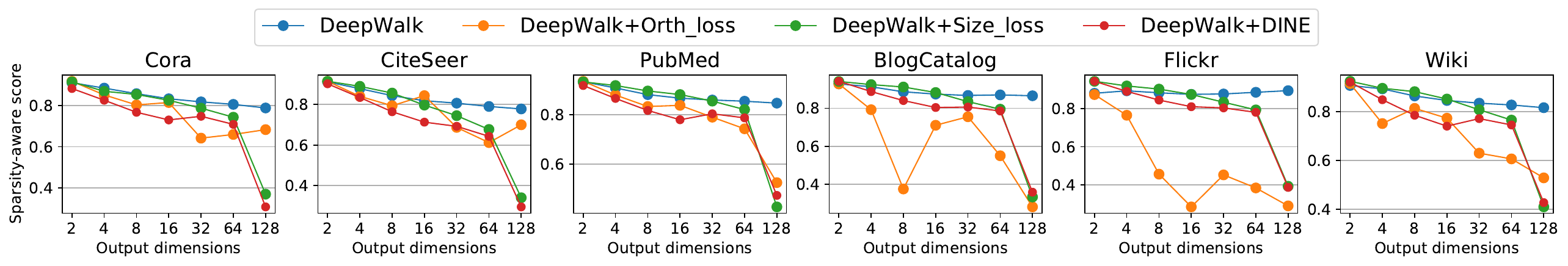}
        \caption{Interpretability scores for the DINE embeddings trained on \textsc{DeepWalk} vectors when removing regularizers (\textit{Size Loss}, \textit{Orthogonality Loss} or both), varying the number of output dimensions and choosing the best score among models with different number of input dimensions. In the top panel, we compare the community-aware scores; in the bottom panel, we compare sparsity-aware scores.}
        \label{fig:abl}
\end{figure*}

\newpage

\subsection*{Noise Robustness Studies}

In Figure \ref{fig:noise} we show the interpretability metrics on multiple datasets when adding random noise to the entries of \textsc{DeepWalk} embedding matrices. For each entry $X_{d,v} = \mathrm{v}_d$ we sample a random Gaussian variable $\epsilon_{d,v} \propto \mathcal{N}(0, \delta)$, where the variance $\delta$ represents the amount of noise that we want to add. Then, we perturb the original embedding entries with the following rule $\Tilde{X}_{d,v} = X_{d,v} \cdot \exp{\epsilon_{d,v}}$.

\begin{figure*}[h!]
     \centering
         \centering
         \includegraphics[width=\textwidth]{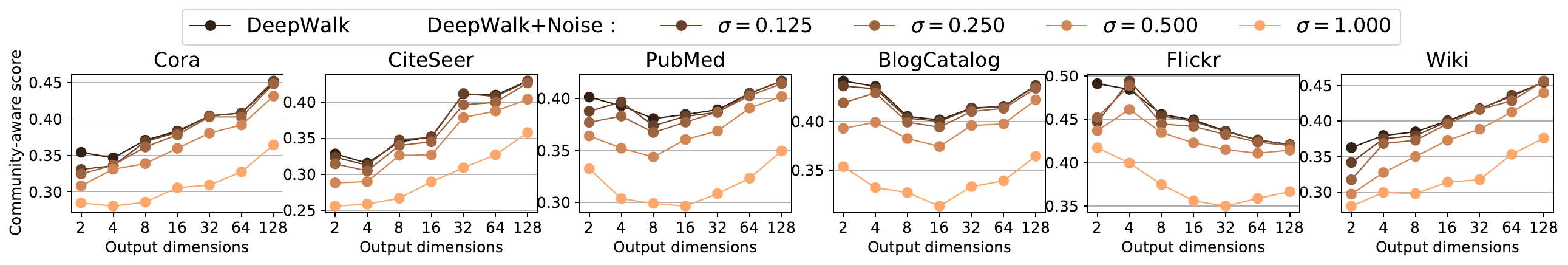}
         \centering
         \includegraphics[width=\textwidth]{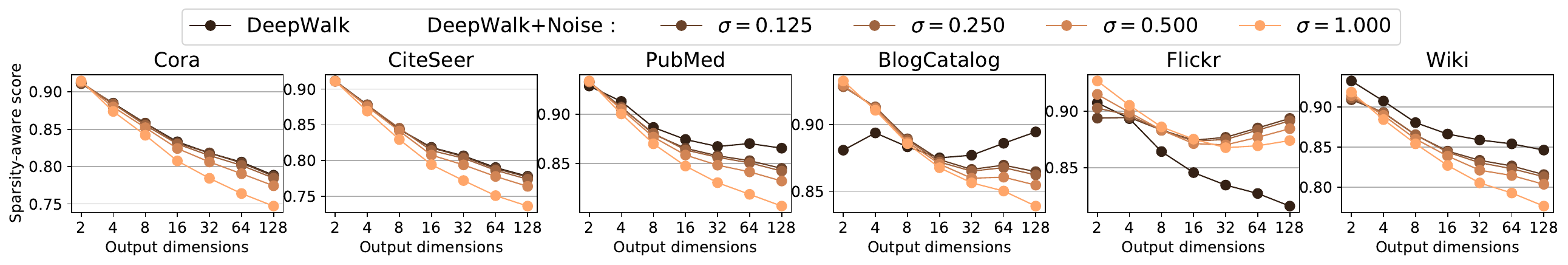}
        \caption{Interpretability scores for the \textsc{DeepWalk} embeddings perturbed with variable entry-wise noise, varying the number of output dimensions and choosing the best score among models with different number of input dimensions. In the top panel, we compare the community-aware scores; in the bottom panel, we compare sparsity-aware scores.}
        \label{fig:noise}
\end{figure*}
    
\vfill

\end{document}